%% file: root.tex
\documentclass[
    Afour,
    sageh,
    times
]{sagej}
\usepackage{moreverb}
\usepackage{bm}
\usepackage{url}
\usepackage{multirow}
\usepackage[colorlinks,bookmarksopen,bookmarksnumbered,allcolors=blue]{hyperref}
\usepackage[caption=false,font=footnotesize]{subfig}
\usepackage{algorithm}
\usepackage[noend]{algpseudocode}
\input{notation.tex}

\newcommand\BibTeX{{\rmfamily B\kern-.05em \textsc{i\kern-.025em b}\kern-.08em
T\kern-.1667em\lower.7ex\hbox{E}\kern-.125emX}}

\begin{document}

\runninghead{Battocletti \textit{et al.}}

\title{Distributed Motion Planning for Multi-Robot Systems under Topological Constraints}

\author{Gianpietro Battocletti\affilnum{1}, Dimitris Boskos\affilnum{1}, Dimos V. Dimarogonas\affilnum{2}, and Bart De Schutter\affilnum{1}}

\affiliation{\affilnum{1}Delft Center for Systems and Control, Delft University of Technology, Delft, The Netherlands.\\
\affilnum{2}Division of Decision And Control Systems, KTH Royal Institute of Technology, Stockholm, Sweden.}

\corrauth{Gianpietro Battocletti, Delft Center for Systems and Control, 
Delft University of Technology, Mekelweg 2, 2628 CD Delft, The Netherlands.}
\email{g.battocletti@tudelft.nl}

\begin{abstract}
    Efficient and distributed coordination of mobile robots is one of the main challenges in multi-robot systems. Topological constraints, often expressed as topological braids, are a popular tool to encode complex coordination patterns between multiple mobile robots, as they offer a compact and abstract representation of the desired qualitative relation between the space-time trajectories of the robots.
    However, execution of joint motion plans encoded as braid-based topological constraints via distributed controllers is challenging, with existing approaches, generally based on the execution of one braid generator at a time, producing slow and suboptimal trajectories.
    We propose a distributed controller based on Model Predictive Control (MPC) to efficiently execute braid-based topological specifications. 
    Rather than directly tracking the braid specification, we propose to use winding numbers, which are topological invariants for braids, as a proxy. This has the twofold benefit of converting braids into a continuous function, which can be easily tracked by an MPC controller through an appropriate term in the cost function, and of decoupling the global braid specification into a set of pairwise specifications, which can be tracked distributedly through the solution of only local MPC problems.
    To maintain global coordination, we propose a consensus-based progress estimation approach, which allows the robots to synchronize their motion toward the desired specification.
    We validate the proposed approach in simulation and in real-world experiments, where we demonstrate the effectiveness of the proposed approach and the improvement over existing approaches in terms of execution speed and control effort.
\end{abstract}


\maketitle

\section{Introduction}
\label{sec:introduction}
Autonomous multi-robot systems are becoming increasingly widespread, with applications ranging through many different domains. Modern hardware allows for a high degree of autonomy of mobile robots, enabling them to perform on board all the necessary tasks required for autonomous operation, such as mapping, localization, planning, and control. This also allows for multi-robot systems to operate with distributed or decentralized control architectures, increasing their scalability and removing the need for a central coordinator. Under those control architectures, robots sharing the same environment must coordinate with each other in order to avoid collisions, to optimize their motion trajectories, and to perform shared tasks. 
To this end, topological constraints represent an increasingly popular coordination tool \citep{bhattacharya2021guest}, thanks to their ability to compactly represent the qualitative relation between the trajectories of different robots.

In the last decade, the success of topology-based motion planning approaches for single-robot systems has lead to a growing interest in the development of novel methods to extend their use to multi-robot systems.
However, while several algorithms have been proposed to solve multi-robot coordination tasks by computing some type of topological specification, the same does not hold for methods designed for the execution of such specifications. In fact, only a limited number of algorithms exist to execute multi-robot motion plans encoded as topological constraints, and they present several limitations.
The most widely used approach to execute joint motion plans encoded as braid-based specifications and grid-based specifications is that of discretizing the set of positions the robots can take in the environment, and sequentially swapping the relative positions of pairs of robots, which results in practice in the execution of one element of the specification at a time \citep{diazmercado2015correct, diaz2017multirobot,cao2023path, cao2025braid}. 
While this approach is effective, it is slow and scales poorly with the number of robots, as robots only move one pair at a time, and need to wait for others to complete their swap before being able to move. 
Another popular approach is that of selecting a set of trajectories compatible with the desired specification and tracking them \citep{imran2026homotopy,mavrogiannis2023implicit,mavrogiannis2020decentralized}. 
However, this method is sensitive to the sampling strategy, and usually produces suboptimal motion trajectories.
A third possible approach is that of using winding numbers to encode pairwise swaps between agents. The sign of winding numbers indicates if the swap must happen in clockwise or counterclockwise fashion; the robots are then tasked with computing trajectories that maximize the increase of the winding numbers corresponding to the traveled paths, while matching the desired sign \citep{mavrogiannis2022winding, nakao2025symmetry, mavrogiannis2018multi, mavrogiannis2021hamiltonian}. However, algorithms based on this approach are only suitable for the execution of a limited class of topological specifications, composed by those with at most one swap for each pair of robots. Moreover, those approaches provide no guarantees that the desired specification will be executed due to the absence of a temporal ordering of the swaps being encoded in the pairwise winding numbers.

In this work, we aim at bridging this gap by proposing a distributed motion planning algorithm for a multi-robot system capable of executing general topological specifications, i.e., of making so that the trajectories traveled by the robots lie in the topological equivalence class encoded by the specification. 
Without loss of generality, we focus on a specific type of topological constraint, namely, topological braids, as it is the most widely used one in works on multi-robot systems \citep{cao2023path, diaz2017multirobot, mavrogiannis2020decentralized}. 
Moreover, other types of constraints, such as sequences of permutation grids and winding numbers, can be translated into topological braids  \citep{cao2025braid, berger2001topological}, allowing to apply the proposed algorithm with little modifications.

To achieve our goal, we propose the use of winding numbers as a way to translate the braid specification, which is a discrete object by nature, into a continuous one. 
This conversion retains the topological properties of the braid, since winding numbers are topological invariants for braids \citep{berger2001topological}.
We then attain the execution of the specification via a distributed MPC-based braid controller that takes as input the winding numbers obtained from the braids, and computes feasible collision-free trajectories to track them thanks to the inclusion of a winding tracking term in the MPC cost function.
Through this approach, we are able to track general braids, differently from existing works considering a winding tracking controller, which are limited to braids that include at most one swap for each pair robots such as in \citep{mavrogiannis2022winding,mavrogiannis2021hamiltonian}, and where winding numbers are often used solely as a mean to expedite the swap of position between pairs of robots.
One challenge in this approach arises from the fact that braids do not provide a timing specification for the execution of the specification, which can make it challenging to synchronize the robots during their execution, in particular when considering a distributed control architecture. For this purpose, we propose to estimate the progress along the specification through the comparison of the winding numbers associated to the trajectories of the robots with those computed from the braid specification. 
This progress estimate can be computed locally and agreed upon through distributed consensus, and avoids the need for a pre-specified timing law for the tracking of the braid, which in general will be suboptimal if it requires executing the specification too slowly, or unattainable if it requires the execution to be too fast. 

The contributions of this work with respect to the state of the art are summarized as follows:
\begin{itemize}
    \item We introduce a novel approach to translate discrete braid-based topological specifications into continuous functions by means of winding numbers, which yield invariants for topological braids. 
    The use of winding numbers results in splitting the global interaction pattern encoded by the braid in pairwise interaction patterns, which can be tracked distributedly.
    \item We propose a distributed MPC-based control approach to execute braid-based multi-robot topological specifications, based on the tracking of winding numbers. 
    Winding numbers can be directly computed from the trajectories of the robots, allowing to search for motion trajectories that minimize the mismatch of the trajectories' winding numbers from the target winding numbers (i.e., those computed from the braid specification). This approach is scalable and efficient, since it requires the solution of only local MPC problems. 
    We prove that the distributed MPC controller is recursively feasible, ensuring the existence of collision-free trajectories that satisfy state and input constraints.
    \item We introduce a consensus-based progress variable to track the progress of the robots along the topological specification, that also serves as a coordination tool to synchronize the robots toward the desired interaction pattern.
    \item The contributions are validated by means of numerical simulations, in which we show the benefits of the proposed controller with respect to existing approaches, and via real-world experiments, where we assess the robustness and effectiveness of the proposed approach in the presence of disturbances.
\end{itemize}
The remainder of this article is organized as follows.
In Section \ref{sec:related-work} we review existing work about the application of topological specifications to motion planning for mobile robots, and in Section \ref{sec:background} we introduce the background material and the problem setting. 
The proposed approach is composed by a preprocessing phase, presented in Section \ref{sec:preprocessing}, and by an online control phase, presented in Section \ref{sec:distributed-mpc}.
We demonstrate the controller capabilities through numerical simulations in Section \ref{sec:simulations}, and in real-world experiments in Section \ref{sec:real-world-experiments}. Section \ref{sec:conclusions} concludes the paper.

\section{Related work}
\label{sec:related-work}
In this section, we expand on the overview we have provided in Section \ref{sec:introduction} of approaches that compute motion plans for mobile robots and encode them as topological constraints.

In single-robot systems, topological tools allow to qualitatively categorize the possible trajectories of a robot depending on the homotopy class to which they belong. This corresponds to grouping together trajectories that can be continuously deformed into one another, establishing a qualitative equivalence between them \citep{lee2010introduction}.
In practice, the most common approach to determine the homotopy class where a trajectory lies in, is that of computing a topological invariant of the trajectory, with the most widely used being the so-called homotopy signature \citep{bhattacharya2018path}.
Homotopy signatures can be used to identify the homotopy class where a trajectory lies, and can also be used as a topological specification to impose a homotopy class constraint in the motion planning problem for mobile robots \citep{grigoriev1998polytime, kim2014path, yang2022efficient, pokorny2016topological}.
Topological invariants other than signatures have also been considered for this purpose, such as the number of times a trajectory loops around each obstacle in the environment \citep{yershov2013continuous, narayanan2013planning}, and the integral of the trajectory of a robot in the complex plane \citep{he2022homotopy, bhattacharya2012topological, bhattacharya2010search}.
Alternatively, homotopy-class constraints can be cast as mixed-integer constraints \citep{park2015homotopy}, which can be more easily included in optimization-based planners.
In the context of optimization-based motion planning approaches, homotopy-based methods can be used to generate topologically-distinct initial solutions for the trajectory optimization problem. The parallelized optimization of trajectories in different homotopy classes has been shown to lead to motion plans closer to the global optimum \citep{degroot2024topology, wong2025generating, sahin2024topo}.
Another notable application case of homotopy-based topological specifications is that of tethered robots, where the presence of a tethered connection between the robot and a fixed anchor point introduces topological constraints in the motion planning problem.
Several approaches have been proposed to tackle this problem, with the majority of them leveraging topological invariants such as homotopy signatures to search feasible paths in distinct homotopy classes, and to find one where the trajectory of the robot is compatible with the maximum tether length \citep{kim2014path, tanner2013online, teshnizi2021motion}.

Topological constraints in multi-robot systems require more complex topological tools.
In fact, topological constraints in multi-robot systems must capture, in addition to the relations between the trajectories of the robots and the obstacles present in the environment, also the relations between the trajectories themselves.
For this reason, the main challenge these tools need to address is that the classification of multi-robot trajectories in topological equivalence classes jointly depends on the topological relations between the trajectories.
Several types of topological tools have been considered to represent the topological equivalence classes of the trajectories of multiple robots, such as topological braids \citep{mavrogiannis2017socially,diaz2017multirobot, liu2024reasoning}, sequences of permutation grids \citep{cao2025braid}, winding numbers \citep{mavrogiannis2022winding, nakao2025symmetry, ma2025topology}, and representative sets of topologically distinct paths \citep{wu2020multi, wang2023coordination, mcbeth2023scalable}.
Topological specifications based on these topological objects have been shown to be a powerful tool for the coordination of multiple robots, both in centralized \citep{imran2026homotopy}, distributed \citep{mavrogiannis2019multiagent, dong2025homotopy} and decentralized settings \citep{mavrogiannis2023implicit, mavrogiannis2024abstracting}.
Those topology-based methods have been successfully employed in several applications for the coordination of multi-robot systems \citep{mavrogiannis2023implicit}, including crowd navigation \citep{mavrogiannis2019multiagent, baselga2026shine} and traffic control \citep{mavrogiannis2024abstracting, liu2024reasoning, wang2023coordination}.
Similarly to the single-robot case, also multi-robot systems composed of multiple tethered robots represent a prime application case for topological specifications, since the tethers of the robots naturally give raise to multiple topological equivalence classes where each tether can lie with respect to the others. In this context, braid-based planning approaches have been proposed for the entanglement-free navigation of tethered mobile robots \citep{cao2025braid, zhang2019planning}.
Moreover, in \citep{dantonio2026hitch} topological specifications have been used to compute trajectories leading to the formation of knots between the tethers of multiple drones, which can be used to grab and lift a payload for transportation purposes.

Multiple methods have been proposed in the literature for the execution of motion plans encoded through topological constraints, but with several limitations.
The most popular approaches have been reviewed in Section \ref{sec:introduction}, where their limitation have also been highlighted, such as their suboptimality and slow speed in executing the topological specification.
In this work, we propose a distributed control method to overcome the limitations of existing approaches, producing efficient trajectories that satisfy the target topological specification.

\section{Background and problem setting}
\label{sec:background}
In this section we introduce the relevant background on the topological objects considered in this work, namely, braids and winding numbers, and the problem setting. 
We conclude the section with an overview of the proposed approach.

\subsection{Topological braids}
\label{sec:background-braids}
Topological braids are topological objects with algebraic and geometric presentations \citep{artin1947theory, kassel2008braid}. Geometrically, braids are defined as a tuple $\bm{f} = (f_1, f_2, ..., f_n)$ of $n$ strands $f_i:[0,1]\rightarrow\mathbb{R}^2\times[0,1], i\in\{1, ..., n\}$, with $f_i: \tau \mapsto (x(\tau), y(\tau), \alpha(\tau))$, where $\alpha:[0,1]\rightarrow[0,1]$ is a parameterization. 
\begin{assumption}
    \label{assumption:alpha-tau-equal-tau}
    For simplicity, but without loss of generality, in the following we assume $\alpha(\tau)=\tau$. 
\end{assumption}
Strands monotonically increase along $\tau$, do not intersect with each other, i.e., $f_i(\tau) \neq f_j(\tau), \forall i, j \in \{1, ..., n\}, i\neq j, \forall \tau\in[0,1]$, and are defined so that $f_i(0) = (i, 0, 0)$, $f_i(1) = (p(i), 0, 1)$, where $p:\{1, ..., n\}\rightarrow\{1, ..., n\}$ is a permutation of $n$ elements, with $p(i)$ representing the image of $i$ through $p$. 
An example of a geometric braid represented via a so-called braid diagram is shown in Figure \ref{fig:braid}a.
\begin{figure}[t]
    \centering
    \includegraphics[width=\columnwidth]{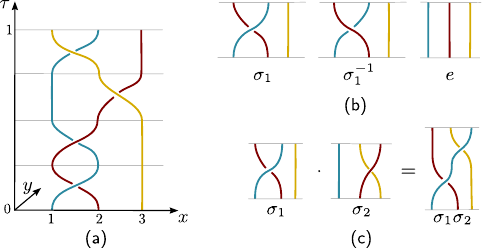}
    \caption{(a)~Example of a geometric braid composed of 3 strands. A corresponding braid word is $\omega=\sigma_1\sigma_1^{-1}\sigma_2\sigma_1^{-1}$. (b)~Two braid generators $\sigma_1$, $\sigma_1^{-1}$, and the identity braid $e$. (c)~Example of product operation between two braid generators, corresponding to their concatenation.}
    \label{fig:braid}
\end{figure}
Geometric braids can be viewed as an ordered sequence of pairwise swaps between adjacent strands.
Under this point of view, they can be represented as the concatenation of the elements of a set of $n-1$ elementary braids $\{\sigma_i\}_{i=1}^{n-1}$, called generators, along with their inverses and the identity braid $e$ \citep{artin1947theory}.
Each generator $\sigma_i$ is a braid of $n$ strands whose diagram swaps the $i$-th and $(i+1)$-th strands. An example of generators is shown in Figure \ref{fig:braid}b, while an example of their concatenation is depicted in Figure \ref{fig:braid}c.
As a result, a geometric braid $\beta$ can be equivalently represented as an algebraic braid through a word $\omega$, corresponding to the product of $B$ elements from the generator set $\{\sigma_q\}_{q=1}^{m-1}$ \citep{artin1947theory}:
\begin{equation}
    \label{eq:braid-word}
    \omega = \prod_{b=1}^B \sigma_{g_b}^{s_b}
\end{equation}
with $g_b\in\{1, \ldots, m-1\}, s_b\in\{-1, +1\}$. The index $g_b$ of a generator indicates which strands are being swapped (namely, strands $g_b$ and $g_b+1$), and the exponent $s_b$ indicates the rotation direction of the swap. The rotation convention is arbitrary; in the following we use $+1$ to indicate a counterclockwise rotation, and $-1$ to indicate a clockwise one.
It is important to note, however, that while $\beta$ corresponds to an equivalence class of geometric braids equivalent up to homeomorphism, $\omega$ encodes only a strict subset of them, since it fixates the order in which the generators must appear.
In order to obtain a one-to-one correspondence between algebraic and geometric braids, the following equivalence relations between braid words are defined \citep{artin1947theory}
\begin{subequations}
    \label{eq:yang-baxter}
    \begin{align}
        &\sigma_{g_1} \sigma_{g_2} = \sigma_{g_2} \sigma_{g_1}, |g_1 - g_2| > 1,\label{eq:yang-baxter-1}\\
        &\sigma_{g_1} \sigma_{g_2} \sigma_{g_1} = \sigma_{g_2} \sigma_{g_1} \sigma_{g_2}, |g_1 - g_2| = 1,\label{eq:yang-baxter-2}
    \end{align}
\end{subequations}%
with the first stating that generators can commute if they do not share a strand, and the second corresponding to a continuous deformation of a triplet of strands \citep{thiffeault2022braids}.
The generator set $\{\sigma_i\}_{i=1}^{n-1}$, together with the product operation corresponding to the concatenation of generators, and with the relations \eqref{eq:yang-baxter}, gives raise to the braid group of $n$ strands $\mathfrak{B}_n$.
The relations \eqref{eq:yang-baxter} also provide a convenient tool for comparing different braids words; in fact, two braids $\beta_1$, $\beta_2$ are equivalent if the corresponding braid words are the same up to the relations \eqref{eq:yang-baxter}.
Given the correspondence between geometric and algebraic braids, braid words can be used to efficiently and compactly represent braids thanks to their discrete nature.
A word $\omega$ is sufficient to identify a topological equivalence class of geometric braids $\beta$. The selection of the representative word $\omega$ can be arbitrary, as all words corresponding to $\beta$ are equivalent under \eqref{eq:yang-baxter}.
However, while deciding equivalence of two braid words is polynomial \citep{epstein1992word}, deciding whether a word has minimal length is co-NP-complete, so finding a shortest representative is NP-hard \citep{paterson1991set}.
\begin{figure}[t]
    \centering
    \includegraphics[width=\columnwidth]{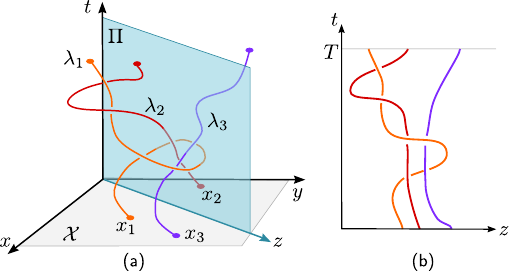}
    \caption{(a) a set of three space-time curves $\lambda_1$, $\lambda_2$, $\lambda_3$ generated by the motion of 3 points $x_1, x_2, x_3$ in $\mathcal{X}$ over time. The curves are projected on the plane $\Pi$ to generate the braid diagram shown in (b). The braid diagram corresponds to the braid word $\omega=\sigma_1^{-1}\sigma_2^{-1}\sigma_2^{-1}\sigma_1^{-1}\sigma_1^{-1}\sigma_1^{-1}$.}
    \label{fig:braid-projection}    
\end{figure}

To further strengthen the connection between the two presentations of braids described above, under Assumption \ref{assumption:alpha-tau-equal-tau} we can assume a uniform spacing of braid generators along geometric braids.
\begin{assumption}
    \label{assumption:spacing-generators}
    In the geometric representation of a braid $\beta$ composed by $B\in\mathbb{N}_{>0}$ generators, each generator $\sigma_{g_b}^{s_b}, b\in\{0, ..., B\}$ spans an interval $[\tau_b, \tau_{b+1}]$ of length $1/B$, with $\tau_b=b/B$, i.e., generators are uniformly spaced along $\tau$.
    Moreover, in all the $B+1$ points $\tau_b$, we have $x_i(\tau_b) \in\{1, ..., n\}$ and $y_i(\tau_b)=0$, $i\in\{1,...,n\}$.
\end{assumption}

\subsection{Topological braids from trajectories}
\label{sec:background-braids-trajectories}
We now introduce an additional interpretation of braids, which connects them to the trajectories of a multi-robot system. 
In fact, braids can be extracted from the set of trajectories $\Lambda=\{\lambda_i\}_{i=1}^m, \lambda_i:[0,T]\rightarrow\mathcal{X}\times[0,1] : t\mapsto (x_i(t), y_i(t), t)$, generated by a set of robots moving in a 2D workspace $\mathcal{X}\subseteq\mathbb{R}^2$ over a time interval $[0,T]$. 
This is achieved by projecting the set of curves traced by the trajectories $\lambda_i\in\Lambda$ in the space $\mathcal{X}\times[0,T]$ by the robots as they move, onto a plane $\Pi$ defined by an axis $z$ lying in the $x$-$y$ plane and by the time axis $t$ \citep{mavrogiannis2023implicit,thiffeault2010braids}. 
An example is shown in Figure \ref{fig:braid-projection}a.
We highlight that the projection plane $\Pi$ can be selected arbitrarily, as braids obtained from different projection planes are related by conjugation \citep{thiffeault2010braids}.\endnote{Two braids $\beta_1$, $\beta_2$ are said to be conjugate if there exists a braid $\zeta\in\mathfrak{B}_n$ such that $\beta_1 = \zeta\beta_2\zeta^{-1}$ \citep{thiffeault2010braids, thiffeault2022braids}. 
For more details on existing methods to evaluate if two braids are conjugate, and to compute a conjugacy braid $\zeta$, see \citep{garside1969braid, thiffeault2014braidlab}.}
The result of this operation is a braid diagram, which can be equivalently expressed as a braid word by concatenating the braid generators corresponding to the swaps between trajectories, following the temporal ordering in which they appear along the time axis $t$ (see Figure \ref{fig:braid-projection}b).

\subsection{Winding numbers}
\label{sec:background-windings}
Given a braid diagram such as the one represented in Figure \ref{fig:braid}, for each $\tau\in[0,1]$, the winding number \citep{berger2001topological} of a pair of strands $i, j\in\{1,...,m\}$ is defined as:
\begin{equation}
    \label{eq:winding-continuous}
    w_{i,j}(\tau) = \frac{1}{2\pi}\int_{0}^\tau \dot{\theta}_{i,j}(\tau') d\tau',
\end{equation}
with $\theta_{i,j}(\tau')$ representing the relative angle between the strands $i$ and $j$ at $\tau'\in[0, \tau]$, as shown in the example of Figure \ref{fig:winding_angles}a.
Winding numbers can be used to encode the `mixing' between two strands of a braid \citep{mavrogiannis2022winding}, with a braid diagram yielding $\binom{m}{2}$ pairwise winding numbers.
For all pairs of strands $i, j \in \{1,...m\}$ it holds that
\begin{equation}
    \label{eq:w_ij_w_ji}
    w_{i,j}(\tau) = w_{j,i}(\tau), \forall i, j \in\{1,...m\}, \forall \tau \in[0,1].
\end{equation}

Each generator in a braid represents a swap between two strands, which corresponds to a rotation of $\pm\pi$. This is independent of the actual path traced by the strands, and leads to a net variation of the winding number value between the beginning and the end of the braid generator equal to $\pm1/2$ (see Figure \ref{fig:winding_angles}b).
Owing to the independence of the winding number value from the strands, winding numbers at $\tau=1$ are topological invariants for braid generators \citep{berger2001topological}.
We highlight that winding numbers defined as in \eqref{eq:winding-continuous} are not invariants for any value of $\tau$, but only at the points where two generators are connected together \citep{berger2001topological, mavrogiannis2022winding}, i.e., at a discrete set of points.
Under Assumptions \ref{assumption:alpha-tau-equal-tau} and \ref{assumption:spacing-generators}, for a braid $\beta$ composed by $B$ generators, the $B+1$ points where this is true are the points $\{b/B\}_{b=0}^B$.
For each $b\in\{1,...,B\}$, winding numbers at $\tau=b/B$ are topological invariants for the sub-braid $\beta_b = (f_1([0, b/B]), f_2([0, b/B]), ..., f_n([0, b/B]))$.
Finally, it is worth noting that winding numbers are not complete topological invariants. 
In fact, while all braids in the same equivalence class give raise to the same set of winding numbers at $\tau=1$, this mapping is not injective, as multiple equivalence classes of braids can correspond to the same set of winding numbers.

\begin{figure}[t]
    \centering
    \includegraphics[width=\columnwidth]{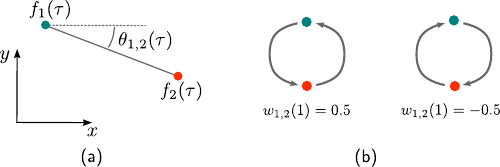}
    \caption{On the left-hand side, an example of the relative angle $\theta_{i,j}$ for a pair of strands. On the right-hand side, the winding number corresponding to a pairwise swap between two strands, the sign of which depends on the rotation direction.} 
    \label{fig:winding_angles}
\end{figure}

\subsection{Problem setting}
\label{sec:problem_setting}

Let $\mathcal{M}=\{1, \ldots, m\}$ be a set of $m$ mobile robots moving inside a 2D obstacle-free environment $\mathcal{X}\subset\mathbb{R}^2$.
Static obstacles can be included in the proposed approach with little modification \citep{cao2025braid}.
Each robot $i\in\mathcal{M}$ is initially located at a point $x_i\in\mathcal{X}$, and is tasked to reach a goal point $x_{\mathrm{goal},i}\in\mathcal{X}$.
We assume the presence of a global high-level planner (e.g., one of those introduced in \citep{diaz2017multirobot, diazmercado2015correct, mavrogiannis2023implicit, cao2025braid}) that provides a topological specification, expressed either as a geometric braid $\beta$, or as a braid word $\omega$ representing a topological braid $\beta$, together with a projection plane $\Pi$.
The plane $\Pi$ is defined by a vector $z$ lying in the $x-y$ plane, and by the $t$ axis, as shown in the example of Figure \ref{fig:braid-projection}.

The task of the planner is to compute a set of robot trajectories $\{\lambda_i\}_{i=1}^m$, $\lambda_i:[0, T] \rightarrow \mathcal{X}$ for some common $T\in\mathbb{R}_{>0}$ such that at the end of the motion, (i) each robot reaches its desired goal, and (ii) the projection of the trajectories traversed by the robots on $\Pi$ yields a braid equivalent to the target topological braid $\beta$.

\begin{assumption}
    \label{assumption:projection-to-distinct-points}
    Without loss of generality\endnote{Since the number of robots is finite, if Assumption \ref{assumption:projection-to-distinct-points} is not satisfied for a given plane $\Pi$, it is sufficient to perturb the orientation the plane to obtain a new one for which the assumption holds.} we assume that the projections of the initial locations of the robots and of their goal locations on the plane $\Pi$ are distinct, i.e., that
    \begin{equation}
        \begin{aligned}
            x_i^\Pi&\neq x_j^\Pi\\
            x_{\mathrm{goal}, i}^\Pi &\neq x_{\mathrm{goal}, j}^\Pi
        \end{aligned}, \forall i, j \in \mathcal{M}, i\neq j,
    \end{equation}
    where $x_i^\Pi$ indicates the projection of the initial location of robot $i$ on $\Pi$, and $x_{\mathrm{goal}, i}^\Pi$ that of its goal.
\end{assumption}

\begin{assumption}
    We assume the possibility for the robots to exchange information with each other in an all-to-all fashion.
\end{assumption}

\subsection{Overview of the proposed approach}
The proposed approach is designed to tackle the challenges posed by the problem of executing multi-robot motion plans encoded as braid-based topological constraints, which we now recall. 
\begin{itemize}
    \item First, in order to execute the input braid specification, the braid word corresponding to the braid diagram generated by the projection of the robots' trajectories on a plane $\Pi$ (see Figure \ref{fig:braid-projection}), must match the target one, i.e., the two braid words must coincide up to the equivalence relations \eqref{eq:yang-baxter}. However, the discrete nature of braids makes it challenging to directly compute trajectories that match the desired sequence of generators of the braid, as naively projecting the trajectories on $\Pi$ and evaluating the resulting generators is not suitable for online optimization.
    \item Second, topological braids do not specify a timing law, so the robots must determine autonomously their current state along the specification, as well as decide how to continue to progress along the braid. In this sense, we highlight that using a pre-specified timing law, e.g., one that specifies a fixed progress along the specification for each discrete time step, is in general either suboptimal or unattainable, depending on the initial locations of the robots. For instance, if the robots are widely spaced, their will only be able to progress slowly along the specification, as looping around each other takes more time; however, this cannot easily be encoded in a pre-specified timing law.
    \item Third, we want the control approach to be distributed, which requires (i) decoupling the tracking of the global braid specification in locally-trackable specifications, as well as (ii) the presence of some coordination tool to synchronize the robots toward the execution of the desired global interaction pattern.
\end{itemize}

The proposed control scheme is composed of two parts.
The first is a preprocessing phase, during which the topological braid specification is translated into a set of continuous functions encoding the winding numbers between all pairs of robots over time, which is described in Section \ref{sec:preprocessing}.
This is followed by an online control phase, where each robot solves a local MPC problem and executes the first action of the optimal control sequence to track the specification, which is described in Section \ref{sec:distributed-mpc}.

\section{Preprocessing of the topological constraint}
\label{sec:preprocessing}
In this section we present the preprocessing phase of the proposed algorithm.
As mentioned in the problem setting of Section \ref{sec:problem_setting}, the input topological specification is provided either as a geometric braid $\beta$ or as a braid word $\omega$ (expressed as in \eqref{eq:braid-word}), together with a projection plane $\Pi$.
In both cases, in order to distributedly execute the motion plan encoded by the braid, we propose to translate it into a set of continuous functions
\begin{equation}
    \label{eq:winding-set}
    \mathcal{W} = \{w_{i,j}:[0,1]\rightarrow\mathbb{R}, i,j\in\mathcal{M}, i\neq j\}
\end{equation} 
corresponding to the winding number between pairs of braid strands expressed as a function of the progress variable $\tau$ (vertical axis in Figure \ref{fig:braid}a), computed using \eqref{eq:winding-continuous}.
The values of the winding numbers are initially set to 0 for all the robots pairs, i.e.,
\begin{equation}
    w_{i, j}(0) = 0, \forall i, j \in \mathcal{M}, i \neq j.
\end{equation}

In order to obtain the set $\mathcal{W}$, the topological constraint is run through a preprocessing phase.  
We consider here two cases, namely, the preprocessing of a braid word $\omega$, and that of a geometric braid $\beta$ (encoded as a collection of paths representative of a given topological equivalence class).
The first case is illustrated in Section \ref{sec:preprocessing-word}, and the second one in Section \ref{sec:preprocessing-geometric-braid}.
The output of the preprocessing phase, and the initialization of the robots' winding numbers used in the online control phase, is discussed in Section \ref{sec:preprocessing-output}.

\subsection{Preprocessing of braid words}
\label{sec:preprocessing-word}
As discussed in Section \ref{sec:background-braids}, topological braids can be conveniently expressed as a braid word $\omega$.
In such case, the winding number set $\mathcal{W}$ can be computed by using the following procedure. 

\subsubsection{Winding numbers from braid words}
\label{sec:winding-from-braids}
Given a braid word \eqref{eq:braid-word} composed by $B$ generators, the winding numbers at $\tau_b = b/B, b\in\{1,\ldots, B\}$, are found as
\begin{equation}
    \label{eq:windings-from-generators}
    \begin{aligned}
        &w_{i,j}\left(\tau_b\right)=w_{i,j}\left(\tau_{b-1}\right) \\ 
        &\quad + \begin{cases}
            s_b\dfrac{1}{2} & \text{if } |p_b(i)-p_b(j)|=1 \land p_b\left(\min\{i,j\}\right)=g_b\\
            0 & \text{otherwise}\\
        \end{cases},
    \end{aligned}
\end{equation}
where $p_b:\mathcal{M}\rightarrow\mathcal{M}$ indicates the permutation of the points $i\in\mathcal{M}$ after $b$ braid generators, and is recursively defined as
\begin{equation}
    \label{eq:permutation}
    p_b(i) = \begin{cases}
        p_{b-1}(i+1) & \text{ if } i=g_b,\\
        p_{b-1}(i-1) & \text{ if } i=g_b+1,\\
        p_{b-1}(i)   & \text{ otherwise},
    \end{cases}
\end{equation}
with $p_0(i) = i, \forall i\in\mathcal{M}$.
Under Assumptions \ref{assumption:alpha-tau-equal-tau}, \ref{assumption:spacing-generators}, and \ref{assumption:projection-to-distinct-points}, we can assume without loss of generality that winding numbers grow linearly.
\begin{assumption}
    \label{assumption:linearly-growing-winding-numbers}
    The value of a winding number at a generic point $\tau\in[b/B, (b+1)/B]$ is defined as the linear combination
    \begin{equation}
        \label{eq:windings-linear-segments}
        w_{i,j}(\tau) = (1-\tau') w_{i,j}\left(\frac{b}{B}\right) + \tau' w_{i,j}\left(\frac{b+1}{B}\right),
    \end{equation}
    with $\tau'=\tau B-b$, as shown in the example of Figure \ref{fig:windings}b.    
\end{assumption}
This assumption does not restrict the applicability of the proposed approach, is coherent with the monotonic growth of braids along $\tau$, and 
gives rise to the convenient relation
\begin{equation}
    \label{eq:windings-lipschitz}
    \bigg|w_{i,j}(\tau_2) - w_{i, j}(\tau_1)\bigg| \leq \frac{B}{2}\bigg|\tau_2 - \tau_1\bigg|,
\end{equation}
which holds for any $\tau_1, \tau_2 \in [0,1]$.
This bound follows from the linear interpolation of the winding numbers and from the fact that by \eqref{eq:winding-continuous}, a braid generator gives raise to a change in the corresponding winding number having magnitude equal to $1/2$.
A proof of \eqref{eq:windings-lipschitz} is given in Appendix \ref{appendix:bound-windings}.
\begin{figure}[t]
    \centering
    \includegraphics[width=\columnwidth]{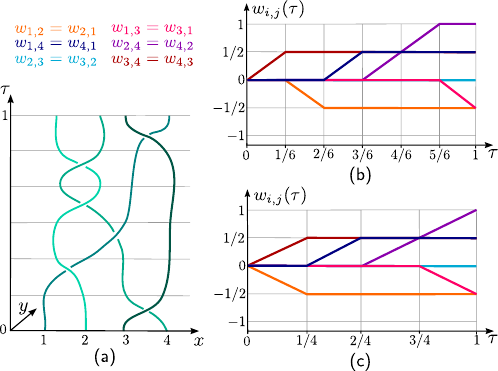}
    \caption{(a)~Example of a geometric braid composed by 4 strands. A corresponding braid word is $\omega = \sigma_3\sigma_1^{-1}\sigma_2\sigma_1\sigma_1\sigma_3^{-1}$. (b)~Winding numbers between each pair of strands in the given braid computed using the preprocessing strategy from Section \ref{sec:winding-from-braids}. (c)~Winding numbers computed from Algorithm \ref{alg:preprocessing}. In (b) we have $B=6$, in (c) $B'=4$.}
    \label{fig:windings}
\end{figure}

We note that, by translating the braid word $\omega$ into a function encoding the winding numbers along $\tau$, the invariance property of winding numbers is lost, as the set $\mathcal{W}$ has a one-to-one correspondence with the specific braid word $\omega$.
This is due to the fact that the winding numbers in $\mathcal{W}$ enforce a specific ordering of the braid generators, not allowing for commutations that lead to equivalent braid words, i.e., those of \eqref{eq:yang-baxter}. As a result, a set of winding numbers encoding a braid word $\omega$ is not representative of a geometric braid $\beta$, but only of some of its realizations.
While, in principle, this is sufficient for the achievement of the objective stated in Section \ref{sec:problem_setting} (and is, in fact, the approach followed by all  existing approaches, such as those proposed in \citep{cao2025braid, diaz2017multirobot}), it is also, in general, unnecessarily restrictive, as according to the problem statement of Section \ref{sec:problem_setting}, any braid equivalent to $\beta$ is suitable for the solution of the problem.
For this reason, in the next section we propose a refined preprocessing routine that provides more flexibility for the specification tracking, but that still results in a braid equivalent to $\beta$.
If, on the contrary, a given task requires the execution of a specific braid word $\omega$, and not of one equivalent to it, the preprocessing routine proposed in this section is sufficient.

\subsubsection{Compacting of the winding numbers sequences}
The braid relation \eqref{eq:yang-baxter-1} allows the commutation of generators that do not share a strand, meaning that the two generators are independent of each other, as they regard distinct pairs of strands.
We can leverage this property of the braid group to further `compact' the winding numbers sequence, so that more generators can be executed in parallel, thus increasing the speed at which the braid can be tracked.
Moreover, this has the advantage of allowing the execution of braids that are equivalent to $\beta$, but not identical to it, making the execution phase more flexible, and potentially resulting in more efficient trajectories being computed. 

We propose to shift parts of the winding numbers functions $w_{i,j}(\tau)$ so that changes in the winding numbers between distinct pairs of robots, which correspond to adjacent pairs of generators between which \eqref{eq:yang-baxter-1} holds, happen at the same point $\tau\in[0,1]$, and not sequentially as one would have from \eqref{eq:windings-from-generators}.
We achieve this through the preprocessing procedure described in Algorithm \ref{alg:preprocessing}.
We note that, conceptually, this compacting routine is similar to that used in \citep{cao2025braid}, with the main difference lying in the nature of the topological object being considered.
\begin{algorithm}
	\caption{Preprocessing routine to compute $\mathcal{W}$}
	\label{alg:preprocessing}
	\begin{algorithmic}[1]
		\Statex \textbf{Inputs}: A braid word $\omega=\prod_{b=1}^B \sigma_{g_b}^{s_b}$, the robots set $\mathcal{M}$
        \State Define $\mathcal{W}'=\{w'_{i,j}: \mathbb{N}_{\geq0} \rightarrow \mathbb{R}, i\in\mathcal{M}, j\in\mathcal{M}, i\neq j\}$
        \State Set $w'_{i,j}(n)=0, \forall n\in\mathbb{N}_{\geq0}, \forall i, j\in\mathcal{M}, i\neq j$
        \State Initialize $B'\leftarrow0, b\leftarrow 1, b'\leftarrow 1, \Omega\leftarrow\emptyset$
        \State Initialize permutation $p_0:\mathcal{M}\rightarrow\mathcal{M}, p_0:i\mapsto i$
        \While{$b\leq B$}\label{line:beginning-main-loop}
            \State $\Omega \leftarrow \Omega \cup \{b\}$
            \State $b'\leftarrow b+1$
            \While{$b'\leq B$}\label{line:beginning-first-inner-loop}
                \For{$b''\in\Omega$}
                    \If{$|g_{b'}-g_{b''}|=1$}
                        \State Go to line \ref{line:end-of-loop}
                    \EndIf
                \EndFor
                \State $\Omega \leftarrow \Omega \cup \{b'\}$\label{line:update-omega}
                \State $b'\leftarrow b'+1$\label{line:end-first-inner-loop}
            \EndWhile
            \State $B'\leftarrow B'+1$
            \For{$b''\in\Omega$}\label{line:end-of-loop}
                \State Update the permutation $p_{B'}$ with \eqref{eq:permutation}
                \State Update the winding numbers as:\label{line:update-windings}
                \vspace{-0.25cm}
                \begin{equation*}
                    \begin{aligned}
                        \quad\qquad w'_{i,j}(B') &\leftarrow w'_{i,j}(B'-1) \\ &+ \begin{cases}
                            s_{b''}\frac{1}{2} & \text{if } |p_{B'}(i)-p_{B'}(j)|=1 \\ 
                            & \; \text{and } p_{B'}\left(\min\{i,j\}\right)\in \Omega,\\
                            0 & \text{otherwise}
                        \end{cases}
                    \end{aligned}
                \end{equation*}
                \vspace{-0.35cm}
                \State $\Omega\leftarrow\emptyset$\label{line:end-second-inner-loop}
            \EndFor
            \State $b \leftarrow b'$ \Comment{Move to new generator and continue} \label{line:end-main-loop}
        \EndWhile
        \For{$b'\in\{0,...,B'-1\}$} \label{line:beginning-second-loop}
        \State Compute final winding numbers as:
        \vspace{-0.25cm}
        \begin{equation*}
            \begin{aligned}
                &w_{i,j}(\tau) \leftarrow (1-\tau') w'_{i,j}\left(b'\right) + \tau' w'_{i,j}\left(b'+1\right),\\
                & \forall \tau\in[b'/B', (b'+1)/B'], \text{ with } \tau'=\tau B'-b'
            \end{aligned}
        \end{equation*}
        \vspace{-0.35cm}
        \EndFor
        \State \textbf{return} $\mathcal{W}=\{w_{i,j}, i, j \in \mathcal{M}, i\neq j\}$ \label{line:end-second-loop}
	\end{algorithmic}
\end{algorithm}

In Algorithm \ref{alg:preprocessing}, a temporary set of winding numbers $\mathcal{W}'$ is used to store the compacted winding numbers. The algorithm iterates through the braid generators (lines \ref{line:beginning-main-loop} -- \ref{line:end-main-loop}), keeping track of the sequence of adjacent generators that do not share a strand (lines \ref{line:beginning-first-inner-loop}--\ref{line:end-first-inner-loop}), and collecting them in the set $\Omega$ (line \ref{line:update-omega}). 
The changes in winding numbers corresponding to the generators in $\Omega$ are then added at the same point in $\mathcal{W}'$ (lines \ref{line:end-of-loop}--\ref{line:end-main-loop}), before resetting $\Omega$ and moving forward along the braid word (lines \ref{line:end-second-inner-loop}, \ref{line:end-main-loop}).
In this sense, the key difference of Algorithm \ref{alg:preprocessing} with respect to the procedure described in Section \ref{sec:winding-from-braids}, lies in the condition $p_{B'}\left(\min\{i,j\}\right)\in \Omega$ in the equation
\begin{equation}
    \begin{aligned}
        w'_{i,j}(B') &\leftarrow w'_{i,j}(B'-1) \\ &\quad+ \begin{cases}
            s_{b''}\frac{1}{2} & \text{if } |p_{B'}(i)-p_{B'}(j)|=1 \\ 
            & \; \text{and } p_{B'}\left(\min\{i,j\}\right)\in \Omega,\\
            0 & \text{otherwise}
        \end{cases}
    \end{aligned}
\end{equation}
at line \ref{line:update-windings}, which, differently from the corresponding condition in \eqref{eq:windings-from-generators}, enables multiple winding numbers to change in the same interval of $\tau$.
Once the end of the braid word is reached, the temporary winding numbers $\mathcal{W}'$ are remapped to the continuous interval $[0,1]$ by rescaling them and linearly interpolating their value in between the points where they are defined (lines \ref{line:beginning-second-loop}--\ref{line:end-second-loop})
An example of the winding numbers resulted from Algorithm \ref{alg:preprocessing} is shown in Figure \ref{fig:windings}c.

The winding numbers $\mathcal{W}$ computed from Algorithm \ref{alg:preprocessing} divide the $[0,1]$ interval into $B'\leq B$ intervals.
By assuming that each interval of $\tau$ takes the same amount of time to be executed, the winding numbers computed from Algorithm \ref{alg:preprocessing} can be tracked quicker than those returned from Section \ref{sec:winding-from-braids}, thanks to the parallelization of the execution of multiple braid generators obtained through the inclusion of the equivalence relation \eqref{eq:yang-baxter-1} in Algorithm \ref{alg:preprocessing}.
Moreover, differently from the preprocessing routine of Section \ref{sec:winding-from-braids}, multiple (equivalent) braid words map to the same set of winding numbers through Algorithm \ref{alg:preprocessing}; for example, the words $\omega_1 = \sigma_1\sigma_3$ and $\omega_2 = \sigma_3\sigma_1$ yield the same set of winding numbers $\mathcal{W}_1$, $\mathcal{W}_2$ through Algorithm \ref{alg:preprocessing}.
This has the advantage of reducing the influence on the algorithm of the representative word $\omega$ provided as input to represent the equivalence class of the geometric braid $\beta$.

\begin{remark}
    We highlight that the winding number set $\mathcal{W}$ obtained from Algorithm \ref{alg:preprocessing} are computed by taking into account the equivalence relation (2a), but not (2b).
    Therefore, in principle $\mathcal{W}$ could be further refined to make the execution of the specification even more parallelized and efficient.
    However, the inclusion of \eqref{eq:yang-baxter-2} in Algorithm \ref{alg:preprocessing} is not trivial, and is left for future work.
    Nonetheless, we remark that (i) the proposed approach still improves upon the existing literature, where only the input braid word $\omega$ can be realized, and (ii) the set of braid words corresponding to $\mathcal{W}$ depends on the representation of $\beta$ provided as input (i.e., on the input braid word $\omega$), which is determined by the external high-level planner computing it. As already mentioned, searching for a shorter braid word is in general NP-hard, and is therefore not considered in the proposed preprocessing phase.
    Lastly, (iii) for the objective defined in the problem setting in Section \ref{sec:problem_setting}, i.e., the execution of a braid equivalent to $\beta$, the use of the winding numbers returned by Algorithm \ref{alg:preprocessing} is sufficient, as they correctly encode $\beta$.
\end{remark}

\subsection{Preprocessing of geometric braids}
\label{sec:preprocessing-geometric-braid}
A topological constraint encoded as a geometric braid $\beta$ is generally provided as a set of paths
\begin{equation}
    \label{eq:representative-paths}
    \Lambda = \{\lambda_i:[0,1]\rightarrow\mathcal{X}\times[0,1]\}
\end{equation}
that yields the geometric braid $\beta$ through their projection on the plane $\Pi$.
Namely, these paths represent a possible realization of the strands of the geometric braid $\beta$.
Two options are available to compute the set of winding numbers $\mathcal{W}$ of \eqref{eq:winding-set} from the set of paths $\Lambda$.
The first option is to compute $\mathcal{W}$ by using \eqref{eq:winding-continuous} for each pair of paths from $\Lambda$, which result in a direct computation of the pairwise winding numbers $w_{i,j}$.
The second option is to compute a braid word $\omega$ associated to the braid diagram obtained by projecting the paths in $\Lambda$ onto the plane $\Pi$. Then, $\mathcal{W}$ can be computed by following the preprocessing procedure from Section \ref{sec:preprocessing-word}. 

There are several ways to describe a set of representative paths $\Lambda$.
One popular way to do so is to encode them as a sequence of permutation grids \citep{cao2023path, cao2025braid, dantonio2026hitch}, compactly describing the paths by keeping track of the relative position of the robots over time.
A permutation grid is defined as a matrix $\mathcal{G} \in \mathcal{M}^{m\times m}$ with all entries equal to zero, except for $m$ entries that have values corresponding to the elements of $\mathcal{M}$. Each row and each column only contains exactly one nonzero element, and each index $i\in\mathcal{M}$ appears exactly once.
Under this definition, a permutation grid $\mathcal{G}$ can be used to represent the relative position that each robot occupies relative to the others \citep{cao2025braid}. 
In particular, if a robot $i\in\mathcal{M}$ has a lower row index in $\mathcal{G}$ with respect to another one, it is closer to the origin along the $x$ axis; analogously, column indexes indicate the relative position of the robots along the $y$ axis, as depicted in the top-left corner of Figure \ref{fig:permutation-grids}a. 

Then, a set of qualitative interactions between the trajectories of the robots can be described as a sequence of grids
\begin{equation}
    \label{eq:grid-sequence}
    \bm{\mathcal{G}} = (\mathcal{G}_1, \ldots, \mathcal{G}_G),
\end{equation}
with $G\in\mathbb{N}_{>0}$.
Each pair of grids $\mathcal{G}_g$, $\mathcal{G}_{g+1}, g\in \{1, ..., G-1\}$, is characterized by the swap of two adjacent rows or of two adjacent columns, corresponding to a swap in the relative position of the two robots occupying those rows/columns.
An example of a sequence of permutation grids is shown in Figure \ref{fig:permutation-grids}a.
\begin{figure}[t]
    \centering
    \includegraphics[width=\columnwidth]{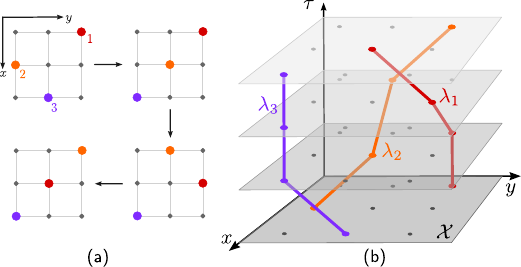}
    \caption{(a)~Example of a sequence of four permutation grids, describing three swaps in the relative positions for a set of $m=3$ robots. (b)~Example of a set of representative paths describing the same topological equivalence class as the sequence of permutation grids on the left.}
    \label{fig:permutation-grids}
\end{figure}
The sequence of relative positions of the robots over $\bm{\mathcal{G}}$ indicates a temporal ordering (but not a timed one) in which the robots must occupy that given set of relative positions; this encodes a topological interaction pattern akin to that of a topological braid.
To construct a set of representative paths $\Lambda$ corresponding to the topological equivalence class described by the sequence $\bm{\mathcal{G}}$, one possible approach is to interpolate the positions that the robots occupy over time, as in the example of Figure \ref{fig:permutation-grids}b.

\subsection{Initialization of the robots' winding numbers}
\label{sec:preprocessing-output}
In the preprocessing phase, the topological constraint, expressed either as a braid word $\omega$ or as a geometric braid $\beta$, is converted in a set of winding numbers $\mathcal{W}$, by using the preprocessing routine described in Section \ref{sec:preprocessing-word} or the one described in Section \ref{sec:preprocessing-geometric-braid}, respectively.
At the end of the preprocessing phase, each robot $i\in\mathcal{M}$ stores in its memory the set of winding number functions
\begin{equation}
    \label{eq:target-windings-set}
    \mathcal{W}_i = \{w_{i,j} \in \mathcal{W}, j\in\mathcal{M}\setminus\{i\}\}\subset\mathcal{W}
\end{equation}
composed by all the winding numbers in which it is involved.

In the online phase of the proposed control approach, the sets of target winding numbers $\mathcal{W}$ are compared with those computed from the trajectories of the robots.
In this work we consider discrete-time trajectories in which we use the variable $k\in\mathbb{N}_{\geq 0}$ to index discrete time steps of $\Delta t$ time units.
The winding numbers corresponding to the trajectories traveled by the robots up to time step $k$ are denoted as 
\begin{equation}
    \widebar{\mathcal{W}}(k) = \{\widebar{w}_{i,j}(k), i,j\in\mathcal{M}, i\neq j\},
\end{equation}
with each element corresponding to a different pair of robots, and the relation \eqref{eq:w_ij_w_ji} holding for each pair $\widebar{w}_{i,j}$, $\widebar{w}_{j,i}$.
The values of the elements of $\widebar{\mathcal{W}}$ depend on the trajectories traveled by the robots, and therefore are time-varying.
In this section we focus on computation of the initial value of $\widebar{\mathcal{W}}$; a detailed explanation of the online update of $\widebar{\mathcal{W}}$ during the motion of the robots is provided in Section \ref{sec:distributed-mpc}.

In principle, the winding numbers set $\widebar{\mathcal{W}}$ could be initialized with a value equal to 0. 
However, when the target winding numbers are computed from a braid word $\omega$, it is implicitly assumed that the robots are initially located at points lying along the $z$ axis, i.e., on the plane $\Pi$. 
Since the initial locations of the robots can be general points in $\mathcal{X}$, this creates an offset between the target winding numbers and those computed from the robots' trajectories. Therefore, when the input to the algorithm is a braid word $\omega$ we initialize the set $\widebar{\mathcal{W}}$ as
\begin{equation}
    \label{eq:angles-init}
    \widebar{w}_{i,j}(0) = \frac{1}{2\pi}\Bigl(\theta_{i,j}(0) - \theta^{\Pi}_{i,j}(0)\Bigr), \forall i, j \in \mathcal{M}, i\neq j,
\end{equation}
where $\theta_{i,j}(0)$ indicates the relative angle between the initial positions $x_i(0), x_j(0)$ of the robots $i$ and $j$, while $\theta^{\Pi}_{i,j}(0)$ is the relative angle between the projections of the points $x_i(0), x_j(0)$ on the plane $\Pi$.
An example of the projection is shown in Figure \ref{fig:projection-init}a, with the relevant variables from \eqref{eq:angles-init} depicted in Figure \ref{fig:projection-init}b.

\begin{figure}[t]
    \centering
    \includegraphics[width=\columnwidth]{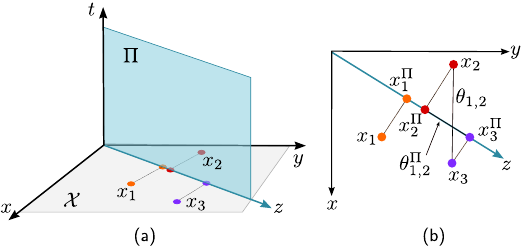}
    \caption{(a)~Projection of the initial positions of the robots on the projection plane $\Pi$. (b)~Computation of the relative angles $\theta_{i,j}(0)$ and $\theta_{i,j}^\Pi(0)$ used to initialize the winding numbers set $\bar{\mathcal{W}}$.}   
    \label{fig:projection-init}
\end{figure}

\section{Distributed winding-based MPC controller}
\label{sec:distributed-mpc}
The preprocessing step introduced in Section \ref{sec:preprocessing} results in the braid specification being decoupled in a set of winding numbers $\mathcal{W}$ describing pairwise interactions between the robots. This enables the possibility of tracking winding numbers locally, as each robot is tasked with tracking only the winding numbers in which it is involved. 
In our approach, this is achieved through a distributed MPC-based controller.

\subsection{Online control loop}
At each discrete time step, each robot sequentially performs the following tasks:
\begin{enumerate}
    \item Update of the realized winding numbers according to the change in the relative positions between the robots.  
    \item Estimation of the current progress state along the topological specification from the winding number values corresponding to the trajectories executed by the robots;
    \item Broadcast to other robots of the local progress variable and of the predicted state sequence computed at the previous time step;
    \item Computation of the global progress variable from the local ones via distributed consensus;
    \item Solution of the local MPC problem;
    \item Execution of the first element of the optimal control sequence computed from the MPC problem.
\end{enumerate}
Next, we illustrate in detail each of these steps to provide a complete description of the proposed approach. A summary of the resulting control loop is provided in Section \ref{sec:summary}.

\subsubsection{Computation of winding numbers from the observed trajectories}
\label{subsub:trajectories}
At every time step, each robot measures the state $x_i$ of itself and the state $\{x_j\}_{j\in\mathcal{M}\setminus\{i\}}$ of every other robot (or receives it during the broadcasting step), and stores it in a trajectory vector $\gamma_l$.
The trajectory traveled by robot $l\in\mathcal{M}$ up to time step $k$ is defined as
\begin{equation}
    \gamma_l(k) = [x^\top_l(0), x^\top_l(1), \ldots, x^\top_l(k)]^\top.
\end{equation}
Two trajectories $\gamma_i(k)$, $\gamma_j(k)$ yield a sequence of relative headings $\theta_{i,j}(h), h\in\{0,...,k\}$ between the robots $i$ and $j$. Without loss of generality, in the following we assume all angles $\theta_{i,j}(h)$ to be wrapped to the interval $(-\pi, \pi]$.
Moreover, we assume a minimum distance $d_\mathrm{min} > 0$ between any two robots $i$, $j$ over each discrete time step $k \in \mathbb{N}$, and that the robots move in a straight line over a single discrete time step. Therefore, we get the following rate constraint on the variation of the relative heading:
\begin{equation}
    \label{eq:theta-rate-constraint-pi}
    \left|\theta_{i,j}(h) - \theta_{i,j}(h-1) \right| < \pi.
\end{equation}
Given a sequence of relative angles $\{\theta_{i,j}(h)\}_{h=0}^k$, we define a discrete-time version of winding numbers \citep{mavrogiannis2022winding} as
\begin{equation}
    \label{eq:winding-discrete}
    \widebar{w}_{i,j}(k) = \frac{1}{2\pi}\sum_{h=1}^k \delta\Bigl(\theta_{i,j}(h), \theta_{i,j}((h-1))\Bigr),
\end{equation}
where $\delta:(-\pi, \pi]^2\rightarrow(-\pi, \pi]$ is an angle wrapping function that, given two input angles, returns the signed angle difference between them. 
The function $\delta$ is defined\endnote{The 2-argument arctangent $\arctan2(y, x)$ is a function that, given two arguments $x, y \in \mathbb{R}$, returns an angle $\theta$ in $(-\pi, \pi]$ such that $x = r \cos{\theta}$ and $y = r \sin{\theta}$, with $ r = \sqrt{x^2 + y^2}$. Unlike the standard arctangent, it correctly determines the quadrant of the angle using the signs of both arguments. The 2-argument arctangent can be used to wrap an angle $\theta$ in the interval $(-\pi, \pi]$ via $\arctantwo(\sin\theta, cos\theta)$, and to compute the smallest signed difference between two angles $\theta_1$, $\theta_2$ via $\arctantwo(\sin(\theta_1 - \theta_2), \cos(\theta_1 - \theta_2))\in(-\pi, \pi]$.} as
\begin{equation}
    \label{eq:delta}
    \delta(\theta_1, \theta_2) = \arctantwo\Bigl(\sin(\theta_1-\theta_2),\cos(\theta_1-\theta_2)\Bigr).
\end{equation}
We note that \eqref{eq:winding-discrete} can be conveniently rewritten as an update equation, i.e., at time step $k$ the winding number corresponding to the trajectories $\gamma_i(k), \gamma_j(k)$ can be computed as
\begin{equation}
    \label{eq:winding-discrete-update}
    \widebar{w}_{i,j}(k) = \widebar{w}_{i,j}(k-1) + \frac{1}{2\pi} \delta\Bigl(\theta_{i,j}(k), \theta_{i, j}(k-1) \Bigr).
\end{equation}
We collect the current winding numbers $\widebar{w}_{i,j}$ in which robot $i$ is involved in the set
\begin{equation}
    \widebar{\mathcal{W}}_i(k) = \left\{\widebar{w}_{i,j}(k), j\in\mathcal{M} \setminus \{i\}\right\}\subset \widebar{\mathcal{W}}(k).
\end{equation} 

\subsubsection{Local estimation of the progress variable}
\label{subsub:local-tau}
The robots keep track of the progress along the topological specification through a progress variable $\tau$. A copy $\tau_i$ of the progress variable is computed locally by each robot $i\in\mathcal{M}$ at the beginning of each time step. Each robot $i$ keeps track of the set of winding numbers $\{\widebar{w}_{i,j}(k)\}_{j\in\mathcal{M}\setminus\{i\}}$ and uses them to estimate the progress along the braid $\beta$ by solving the minimization problem
\begin{equation}
    \label{eq:tau-i}
    \tau_i(k) \in \argmin_{\tau'\in\mathcal{T}} \sum_{j\in\mathcal{M}\setminus\{i\}} \Bigl(w_{i, j}(\tau') - \widebar{w}_{i, j}(k)\Bigr)^2.
\end{equation}
Equation \eqref{eq:tau-i} allows achieving the crucial task of linking the progress variable $\tau$ with the time index $k$, thus addressing the absence of an intrinsic timing law associated to the braid specification $\beta$.

Given the fact that $w_{i,j}$ is not monotonic, and the same value $w_{i,j}(\tau)$ can be achieved for multiple values of $\tau\in[0,1]$, the optimization interval 
\begin{equation*}
    \mathcal{T} = [\tau(k-1), \min\{\tau(k-1) + \Delta\tau_\mathrm{max}, 1\}]
\end{equation*}
ensures that the estimate of $\tau_i(k)$ increases monotonically with $k$ or stays constant, while the upper bound $\Delta\tau_\mathrm{max}$ on the variation of $\tau_i$ limits the number of local minima in \eqref{eq:tau-i}. The value of $\Delta\tau_\mathrm{max}$ depends on the dynamics of the robots, on the length of the time step $\Delta t$, and on the number of generators in the braid specification, and can be expressed as
\begin{equation}
    \label{eq:delta-tau-max}
    \Delta\tau_\mathrm{max} = \frac{2}{\pi B} \arcsin\left(\frac{v_\mathrm{max} \Delta t}{d'_\mathrm{min}}\right),
\end{equation}
where $d'_\mathrm{min}$ is an inflated safety distance designed to prevent collisions when two robots move in a straight line over a discrete time step.
The derivation of \eqref{eq:delta-tau-max} and of $d'_\mathrm{min}$ are discussed in Appendix \ref{appendix:delta-tau-max}.

\subsubsection{Broadcasting of local information}
\label{subsub:broadcast}
At every time step $k$, each robot $i\in\mathcal{M}$ shares its predicted future trajectory 
\begin{equation}
    \begin{aligned}
        \widehat{\bm{x}}_i(k) = [x^\top_i(k), &x^{\star,\top}_i(2|k-1), \ldots, \\ &x^{\star,\top}_i(K|k-1), x^{\star,\top}_i(K|k-1)]^\top
    \end{aligned}
\end{equation}
with all the other robots, corresponding to a shifted version of the solution computed by solving the MPC optimization problem at time step $k$. 
The variable $K>1$ represents the control horizon of the MPC controller, and $\bm{x}^\star_i(k-1)$ represents the solution of the MPC problem for robot $i$ at the previous time step.
Moreover, we use the notation $\square(h|k)$ to indicate a variable $\square$, $h$ steps into the prediction window period of the MPC controller at time step $k$. 
We note that the first element of the vector $\widehat{\bm{x}}_i(k)$ coincides with the current state $x_i(k)$, while the last one is repeated twice to obtain a vector with $K+1$ elements, consistent with the length of the MPC prediction window. We collect the predicted trajectories received by robot $i$ in the set 
\begin{equation}
    \widehat{\mathcal{X}}_i(k) = \{\widehat{\bm{x}}_j(k)\}_{j\in\mathcal{M}\setminus\{i\}}.
\end{equation}
Additionally, each robot shares their local copy of the progress variable $\tau_i$ with the other robots.

\subsubsection{Global progress variable via distributed consensus}
\label{subsub:global-tau}
In order to coordinate with each other, the robots must agree on the current progress along the topological specification. To this end, from the local copies of $\tau_i(k)$ they compute the global estimate $\tau(k)$ of the progress variable. Two options are considered here to do so, namely, 
min-consensus, where the global $\tau(k)$ is found as
\begin{equation}
    \label{eq:tau-global-min}
    \tau(k) = \min_{i\in\mathcal{M}}\tau_i(k),
\end{equation}
and average-consensus, where $\tau(k)$ is found as
\begin{equation}
    \label{eq:tau-global-mean}
    \tau(k) = \frac{1}{m}\sum_{i\in\mathcal{M}} \tau_i(k).
\end{equation}
In the first case, the robots wait for the `slowest' progressing robots before moving further, while in the second they average their estimated progress, thus maintaining a higher overall progress speed. 
The use of \eqref{eq:tau-global-min} versus \eqref{eq:tau-global-mean} is decided a priori by the robots, and all the robots use the same strategy to compute the global progress variable $\tau(k)$. 

\subsubsection{Local MPC controller}
\label{subsub:local-mpc}
Equipped with the information computed at the steps described above, each robot computes a trajectory that, paired with those communicated by the other robots, results in a variation of the winding number as close as possible to the desired one. This is achieved by solving a local MPC control problem formulated as follows:
\begin{subequations}
	\label{eq:MPC}
	\begin{align}
        \nonumber J\Bigl(&x_i(k), x_{\mathrm{goal}, i}, \widehat{\mathcal{X}}_i(k), \mathcal{W}_i, \widebar{\mathcal{W}}_i(k), \tau(k)\Bigr) = \\
        &\min_{\bm{x}_i(k), \bm{u}_i(k)} 
            \Biggl[
                  \xi_\mathrm{u} J_{\mathrm{u}, i}(k) 
                + \xi_\mathrm{g}(\tau(k)) J_{\mathrm{g}, i}(k) 
                + \xi_\mathrm{w} J_{\mathrm{w}, i}(k)
            \Biggr] \label{eq:cost-function}\\
        &\;\text{s.t. } \;x_i(0|k) = x_i(k) \label{eq:constraint-initial-conditions-x} \\
        &\qquad \tau_i(0|k) = \tau(k) \label{eq:constraint-initial-conditions-tau} \\
        &\qquad \widehat{w}_{i,j}(0|k) = \widebar{w}_{i,j}(k) \label{eq:constraint-initial-conditions-w}\\
        &\qquad\text{for } h = 0, \ldots, K-1 \nonumber\\
        &\qquad\qquad x_i(h+1|k) = f_i(x_i(h|k), u_i(h|k)) \label{eq:constraint-dynamics}\\
        &\qquad\qquad u_i(h|k) \in \mathcal{U}_i\label{eq:constraint-input}\\
        &\qquad\qquad x_i(h|k) \in \mathcal{X}_i \label{eq:constraint-state}\\
        &\qquad x_i(K|k) \in \mathcal{X}_{\mathrm{term}, i} \label{eq:constraint-terminal}\\
        &\qquad\text{for } h = 0, \ldots, K \nonumber\\
        &\qquad\qquad\text{for } j \in \mathcal{M} \setminus \{i\} \nonumber\\
        &\qquad\qquad\qquad g\Bigl(x_i(h|k), \widehat{x}_j(h|k)\Bigr) \geq 0 \label{eq:constraint-collisions}\\
        &\qquad\text{for } h = 1, \ldots, K \nonumber\\
        &\qquad\qquad \widehat{w}_{i,j}(h|k) =  \widehat{w}_{i,j}(h-1|k) \nonumber\\
        &\qquad\qquad\qquad + \frac{1}{2\pi}\delta\Bigl(\widehat{\theta}_{i,j}(h|k), \widehat{\theta}_{i,j}(h-1|k)\Bigr),\label{eq:constraint-predicted-winding-number}\\
        &\qquad\qquad \tau_i(h|k) = \min\{\tau_i(h-1|k) + \Delta\tau_\mathrm{max}, 1\} \label{eq:constraint-tau}
     \end{align}
\end{subequations}
where $\mathcal{U}_i$ corresponds to an input constraint, $\mathcal{X}_i$ to a state constraint, and $\mathcal{X}_{\mathrm{term}, i}$ to a terminal state constraint.
The function $g$ in \eqref{eq:constraint-collisions} is used to define a collision avoidance constraint between the trajectory of robot $i$ and the one predicted by robot $j$ at the previous time step. Since in general the actual trajectory of robot $j$ is different from that broadcasted in $\widehat{\bm{x}}_j(k)$, some additional steps are required to guarantee the recursive feasibility of \eqref{eq:MPC}, such as the tightening of the collision avoidance constraint \citep{brito2019model, chen2020guaranteed}, the introduction of consistency constraints to bound the maximum change in the trajectory $\bm{x}_i(k)$ from $\bm{x}_i(k-1)$ \citep{wiltz2025parallelized, lyu2019multivehicle}, or the use of slack variable to render the constraint a softened one. Two of these options are considered and discussed in Sections \ref{sec:recursive-feasibility} and \ref{sec:real-world-experiments}.

The cost function \eqref{eq:cost-function} is composed by three terms, namely, the control cost 
\begin{equation}
    J_{\mathrm{u}, i}(k) = \sum_{h=0}^{K-1} u_i^\top(h|k) Q u_i(h|k),
\end{equation}
with $Q\succ 0$ a weight matrix, the goal tracking cost 
\begin{equation}
    \label{eq:goal-cost}
    J_{\mathrm{g}, i}(k) = \sum_{h=1}^{K}\Big|\Big|x_i(h|k) - x_{\mathrm{goal}, i}\Big|\Big|_2^2,
\end{equation}
and the winding tracking cost
\begin{equation}
    \label{eq:winding-cost}
    J_{\mathrm{w}, i}(k) = \!\!\!\sum_{j\in\mathcal{M}\setminus \{i\}}\!\!\! \xi_{\mathrm{w}, i,j}(k) \sum_{h=1}^{K} \Bigl[ w_{i,j}(\tau_i(h|k)) - \widehat{w}_{i,j}(h|k) \Bigr]^2.
\end{equation}
The predicted winding number $\widehat{w}_{i,j}$ that appears in the winding tracking cost \eqref{eq:winding-cost} is computed through \eqref{eq:constraint-predicted-winding-number}, with the initial value being set to the current winding number through \eqref{eq:constraint-initial-conditions-w}.
In \eqref{eq:constraint-predicted-winding-number}, $\widehat{\theta}_{i,j}(h|k)$ is the predicted relative heading between robots $i$ and $j$, $h$ time steps into the MPC prediction window, which depends on the predicted state sequence $\widehat{\bm{x}}_j(k)$ and the optimization variable $\bm{x}_i(k)$.
We remark that, from the point of view of robot $i$, the trajectory $\widehat{\bm{x}}_j$ is fixed (corresponding to solution computed by robot $j$ at the previous time step), while the trajectory $\widehat{\bm{x}}_i$ is optimized locally through the decision variable $\bm{x}_i(k)$.

A weight is associated to each term composing the cost function \eqref{eq:cost-function}. The weight of the control cost term $\xi_\mathrm{u}$ and that of the winding tracking term $\xi_\mathrm{w}$ are constant.
The weight of the goal tracking cost depends on $\tau(k)$ and is given by the sigmoid function
\begin{equation}
    \label{eq:goal-tracking-weight}
    \xi_\mathrm{g}(\tau) = \frac{1}{1+e^{-s(\tau-c)}},
\end{equation} 
with $s, c \in\mathbb{R}_{>0}$ two tuning parameters. A larger $s$ corresponds to a steeper increase of $\xi_\mathrm{g}$ as $\tau$ grows, while $c$ determines the value of $\tau$ at which the function grows from 0 to 1.
The weight of the goal tracking cost is defined in this way to ease the tracking of the topological specification. In fact, the tracking of the winding numbers and the tracking of the goal represent conflicting objectives, since, if the robot goes directly to the goal, it is neglecting the tracking of the topological specification, and vice versa. 
To overcome this issue, \eqref{eq:goal-tracking-weight} effectively disables the goal tracking cost from \eqref{eq:cost-function} until a given value of $\tau$, which depends on the value of $c$, has been reached, corresponding to some portion of the specification having been executed.

Lastly, for each of the components of the winding tracking cost in \eqref{eq:winding-cost}, i.e., for each robot in the set $\mathcal{M}\setminus\{i\}$, we introduce a time-varying weight defined as
\begin{equation}
    \xi_{\mathrm{w}, i, j}(k) = \frac{1}{||x_i(k) - x_j(k)||_2^2}.
\end{equation}
The weight inversely depends on the relative distance between robots $i$ and $j$ at time step $k$. This increases the relative importance of the winding cost for robots that are closer together, since their relative motion has a higher impact on the change in their winding number.

\subsubsection{Execution of control actions}
\label{subsub:execution}
After the solution to \eqref{eq:MPC} is computed, each robot applies the first element of its optimal control sequence $\bm{u}_i^\star$. The robots take the corresponding actions in the environment, and the control loop starts over again.

\subsection{Recursive feasibility of the MPC controller}
\label{sec:recursive-feasibility}
Under a suitable definition of the collision avoidance constraint \eqref{eq:constraint-collisions}, it is possible to show that the problem \eqref{eq:MPC} is recursively feasible, i.e., if the problem has a feasible solution at time step $k$, then it also has a feasible solution that satisfies all the constraints at time step $k+1$.
To this end, we propose to use a buffered Voronoi partitioning of the environment to generate a set of convex constraints that guarantees the recursive feasibility of the local MPC control problems \citep{park2022online, zhou2017fast}. In particular, we adapt the results from \citep{chen2024deadlock} to define a set of separating hyperplanes which, paired with an appropriate terminal constraint, guarantees the recursive feasibility of the distributed MPC controller.
\begin{proposition}
Let the collision avoidance constraint \eqref{eq:constraint-collisions} be defined as
\begin{equation}
    \label{eq:convex-collision-avoidance}
    a^\top_{i,j}(h|k)x_i(h|k) \geq b_{i,j}(h|k)
\end{equation}
with $a_{i,j}(h|k)$ and $b_{i,j}(h|k)$ defined as
\begin{subequations}
    \label{eq:hyperplanes-parameters}
    \begin{align}
    \label{eq:hyperplanes-parameters-a}
    &a_{i,j}(h|k) = \frac{\widehat{x}_i(h|k) - \widehat{x}_j(h|k)}{||\widehat{x}_i(h|k) - \widehat{x}_j(h|k)||},\\
    \label{eq:hyperplanes-parameters-b}
    &b_{i,j}(h|k) = a^\top_{i,j}\frac{\widehat{x}_i(j|k) + \widehat{x}_j(h|k)}{2} + \frac{d'_\mathrm{min}}{2},
    \end{align}
\end{subequations}
and let the terminal constraint $\mathcal{X}_{\mathrm{term}, i}$ in \eqref{eq:constraint-terminal} be defined as a set of fixed points, i.e., 
\begin{equation}
    \label{eq:recursive-feasibility-terminal-constraint}
    \mathcal{X}_{\mathrm{term}, i} = \{x|\exists u\in\mathcal{U}_i : x = f_i(x, u)\}.
\end{equation}
Then, \eqref{eq:MPC} is recursively feasible at time step $k$ by using the shifted solution from the previous time step $k-1$, with the last element of $\bm{x}^\star_i$ and $\bm{u}^\star_i$ repeated to match the length of the MPC horizon.
\end{proposition}

\begin{proof}
The proof follows from that of \citep[Theorem 3]{chen2024deadlock} since by using \eqref{eq:recursive-feasibility-terminal-constraint} in place of \eqref{eq:constraint-initial-conditions-x}, and \eqref{eq:convex-collision-avoidance}, \eqref{eq:hyperplanes-parameters} in place of \eqref{eq:constraint-collisions}, then the constraint \eqref{eq:constraint-dynamics}, \eqref{eq:constraint-input}, \eqref{eq:constraint-state}, \eqref{eq:constraint-terminal}, and \eqref{eq:constraint-collisions} are equivalent to those of the MPC problem considered in \citep{chen2024deadlock}.
Moreover, the constraints \eqref{eq:constraint-initial-conditions-tau}, \eqref{eq:constraint-initial-conditions-w}, \eqref{eq:constraint-predicted-winding-number}, and \eqref{eq:constraint-tau} are always feasible, and the variables whose values they constrain only appear in the cost function of the MPC problem, therefore not hindering the feasibility of \eqref{eq:MPC}.
\end{proof}

We note that, unlike recursive feasibility, stability and convergence of the proposed control scheme are more difficult to prove, and require further analysis. As such, the investigation of these properties is left for future work.

\subsection{Summary of the proposed approach}
\label{sec:summary}
The approach introduced above is suitable for the distributed execution of multi-robot motion plans encoded as braid-based topological specifications. 
This is achieved by addressing the three main challenges posed by topological braids, namely (i) their discrete nature, (ii) the fact that braids encode a global specification that considers all the robots simultaneously, (iii) the fact that braids encode a temporal ordering of the generators, i.e., of the interactions between the robots, but do not specify a timing law to be followed, with the addition of (iv) the use of a distributed control architecture.
The first two challenges are tackled through the use of winding numbers as a proxy to track the specification.
Winding numbers result in a set of continuous objects that can be integrated in the cost function of the local MPC problem, and they encode pairwise interactions, which can be tracked distributedly. The third challenge is addressed through the consensus-based estimation of the progress variable $\tau$, which allows the robots to compute their progress along the braid over time.
The consensus-based computation of $\tau$ also serves as a coordination tool to synchronize the robots, so that no robot can progress too far with respect to the others along its local winding numbers, rendering it possible to use a distributed control architecture.

The resulting control algorithm is summarized in Algorithm \ref{alg:control-loop}, where we report the key actions taken by each robot during each online control loop. 
\begin{algorithm}
	\caption{Control loop for robot $i$.}
	\label{alg:control-loop}
	\begin{algorithmic}[1]
		\Statex \textbf{Inputs}: Set of target winding numbers functions $\mathcal{W}_i$, goal location $x_{\mathrm{goal}, i}$.
		\State Measure $x_i(k)$ and $\{x_j(k)\}_{j\in\mathcal{M}\setminus i}$
        \State For all robots $l\in\mathcal{M}$ (including $i$), update the traveled trajectories with $\gamma_l(k) \leftarrow [\gamma^\top_l(k-1), x^\top_l(k)]^\top$
        \State Update $\widebar{w}_{i,j}(k), \forall j\in\mathcal{M} \setminus \{i\}$ with \eqref{eq:winding-discrete-update}
        \State Compute $\tau_i(k)$ with \eqref{eq:tau-i}
        \State Broadcast $\widehat{\bm{x}}_i(k)$ and $\tau_i(k)$, receive $\widehat{\mathcal{X}}_i$ and $\{\tau_j\}_{j\in\mathcal{M}\setminus\{i\}}$
        \State Compute $\tau(k)$ via \eqref{eq:tau-global-min} or \eqref{eq:tau-global-mean}
        \State Solve \eqref{eq:MPC}
        \State Apply $\bm{u}^*(0|k)$
	\end{algorithmic}
\end{algorithm}

\section{Simulations}
\label{sec:simulations}

The proposed controller is evaluated in simulations, and compared with existing controllers to evaluate the benefits of the proposed approach.
The simulations have been implemented in Python 3.10.12 and run on a laptop with an 11th Gen Intel processor with four i7 cores (3.00 GHz) and 16 GB of RAM. 
In the code, we leverage the \texttt{braidlab} MATLAB package \citep{thiffeault2014braidlab} to compute braids from trajectories, and to compare braids. 
The source code is available at \href{https://github.com/gbattocletti/braid-controller}{https://github.com/gbattocletti/braid-controller}.

\subsection{Evaluated controllers}
In the simulations, the following controllers are evaluated and compared:

\textbf{Grid-based Braid Controller (GBC)}: This controller is adapted from those proposed in \citep{diaz2017multirobot,cao2025braid} to track topological specifications, and is based on a grid-based discretization of the environment.
First, the workspace $\mathcal{X}$ is discretized as a grid $\Phi$ of  $m\times m$ uniformly-spaced points with a minimum distance between each pair of points of at least $d'_\mathrm{min}$.
Then, the braid word $\omega$ is translated in a sequence of 2D permutation grids $\bm{\mathcal{G}} = (\mathcal{G}_1, \ldots, \mathcal{G}_G)$ with $G\in\mathbb{N}_{>0}$, $G\geq B$ (see \citep{cao2025braid} for more details).
As described in Section \ref{sec:preprocessing-geometric-braid}, each pair of grids $\mathcal{G}_g$, $\mathcal{G}_{g+1}$ is characterized by the swap of two adjacent rows or of two adjacent columns, indicating a swap in the relative position of the robots occupying those rows/columns.
Following the algorithm from \citep{cao2025braid}, these swaps are executed sequentially, following their order in $\bm{\mathcal{G}}$, by a setpoint-tracking MPC controller defined as
\begin{subequations}
    \label{eq:MPC-simple}
    \begin{align}
        \nonumber &J\Bigl(x_i(k), \widetilde{x}_i(k)\Bigr) = \\
        &\min_{\substack{\bm{x}_i(k)\\ \bm{u}_i(k)}}
            \Biggl[
                \xi_\mathrm{u} J_{\mathrm{u}, i}(k) 
                + \xi_\mathrm{g} \sum_{h=1}^{K}\Big|\Big|x_i(h|k) - \widetilde{x}_i(k)\Big|\Big|_2^2
            \Biggr]\\
            &\text{ s.t. }\eqref{eq:constraint-initial-conditions-x}, \eqref{eq:constraint-terminal}\\
            &\qquad\text{for } h = 0, \ldots, K-1 \nonumber\\
            &\qquad\qquad\eqref{eq:constraint-dynamics}, \eqref{eq:constraint-input}, \eqref{eq:constraint-state},
    \end{align}
\end{subequations}
where $\widetilde{x}_i(k)$ represent the target location of robot $i$ on $\Phi$, corresponding to the point in $\Phi$ associated to the entry $i$ in the grid $\mathcal{G}_{g+1}$.
This controller is sufficient for this task since each swap entails the motion of only two robots moving along parallel straight lines (from a pair of nodes of the grid to an adjacent pair).
Since points are spaced more than $d_\mathrm{min}'$, the presence of collision-avoidance constraints in \eqref{eq:MPC-simple} is not required.

\textbf{Distributed Winding MPC (D-WMPC)}: 
This controller implements our proposed approach. First, the topological constraint is broadcasted to all the robots, which locally execute the preprocessing routine described in Section \ref{sec:preprocessing} to compute the set of target winding numbers $\mathcal{W}$.
Second, the online control phase is executed. Each robot runs the control loop of Algorithm \ref{alg:control-loop}, and communicates the necessary information with the others to maintain global coordination and collision avoidance, as described in Section \ref{sec:distributed-mpc}.
\bigbreak
We highlight that we do not compare the proposed approach with the controllers introduced in \citep{mavrogiannis2022winding, nakao2025symmetry}, since they are able to track only a limited set of braid-based specifications, and not general ones. In this sense, we remark that our approach generalizes the controllers proposed there as it is applicable to any braid specification.

\subsection{Simulation results}
The approaches are compared for the execution of different topological specifications, with a varying number of robots $m\in\{3, 5, 10\}$. Each approach is simulated with the same initial conditions, with the controller parameters set to the values listed in Table \ref{table:parameters}.
An example of a simulation of the D-WMPC for a set of $m=10$ robots is shown in Figure \ref{fig:simulation-d-bmpc}.
\begin{table}[!ht]
    \centering
    \small
    \caption{Values of the parameters of the D-WMPC and GBC approaches used in the simulations.}
    \label{table:parameters}
    \begin{tabular}{cc}
        \toprule
        \textbf{Parameter} & \textbf{Value} \\ 
        \midrule
        $\Delta t$ & 0.25 s \\
        $K$ & 21 \\
        $\xi_\mathrm{u}$ & 0.05 \\
        $\xi_\mathrm{g}$ & 1 \\
        $\xi_\mathrm{w}$ & 20 \\
        $s$ & 20 \\
        $c$ & 0.95 \\
        $Q$ & $\mathrm{diag}(1, 1)$ \\
        $d_\mathrm{min}$ & 0.7 \\
        $\mathcal{U}_i$ & $[-0.5, 0.5]\times[-0.5, 0.5]$ \\
        $\mathcal{X}_i$ & $[0, 20]\times[0, 20]$\\
        $\mathcal{X}_\mathrm{term}$ & \eqref{eq:recursive-feasibility-terminal-constraint} \\
        \bottomrule
    \end{tabular}
\end{table}

\begin{figure*}[t]
    \centering
    \includegraphics[width=\textwidth]{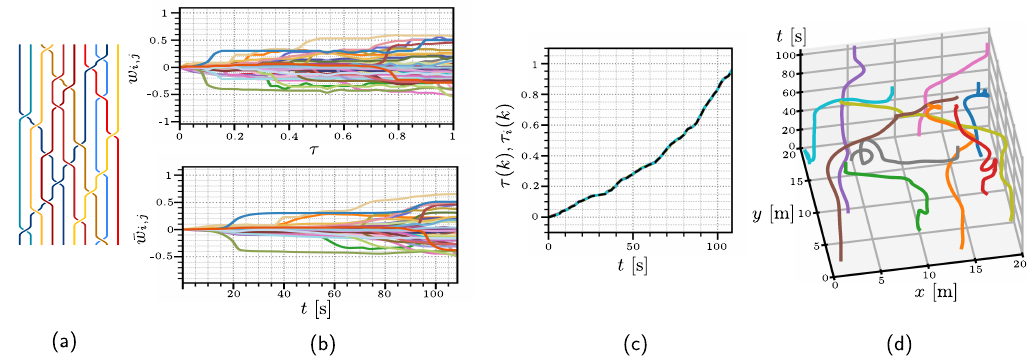}
    \caption{Simulation of the D-WMPC controller for a set of $m=10$ robots. 
    (a)~The braid diagram corresponding to the input specification $\omega$.
    (b)~The target winding numbers $w_ij$ (top plot) and the realized winding numbers $\bar{w}_ij$ corresponding to the trajectories of the robots (bottom plot).
    (c)~The local (colored lines) and global (black dashed line) estimation of the progress variable $\tau$ over time.
    (d)~The trajectories of the 10 robots over time.
    }
    \label{fig:simulation-d-bmpc}
\end{figure*}

For each of the simulated approaches we measure the following evaluation metrics:
\begin{itemize}
    \item The execution time $T$ after which the specification has been fully executed, and the robots have reached their goal locations;
    \item The computation time required to solve \eqref{eq:MPC} (for the D-WMPC approach) and \eqref{eq:MPC-simple} (for the GBC approach), averaged over the $m$ robots and over the $T/\Delta t$ simulation time steps. Additionally, the maximum computation time required to solve the MPC problem in a single time step is reported;
    \item The length $l$ of the trajectories traveled by the robots. We report the trajectories length averaged over the $m$ robots, as well as the length of the shortest trajectory and that of the longest trajectory;
    \item The mean control input $u$, averaged over the $m$ robots and over the $T/\Delta t$ simulation time steps. We also report the cumulative sum of the absolute value of $u$ over the $T/\Delta t$ time steps, which we denote by $\Sigma u$.
\end{itemize}
The results of the simulations are collected in Table \ref{table:simulations}. 
The table shows how the proposed D-WMPC approach outperforms the GBC baseline under most of the evaluation metrics, and achieves a comparable performance for the others. 
In particular, the D-WMPC achieves the execution of all the specifications significantly faster than the GBC baseline, requiring less than $1/3$ of the time necessary for the GBC approach to lead the robots to their goals after executing the specification. 
The computation time of the two approaches is comparable, and both are able to consistently compute a solution to the MPC problem in less than $0.1$s, which is generally suitable for real-time control.
The length of the trajectories traveled by the robots under the D-WMPC approach are also significantly shorter than that of those generated by the GBC baseline, and require a lower control effort, both in terms of average and of total control effort over the simulation time $T$.

Overall, this comparison shows the benefits of the proposed approach over existing ones, and the effectiveness of the D-WMPC approach in coordinating a synchronizing the robots toward the common goal of executing the target braid-based topological specification.

\begin{table*}[!ht]
    \centering
    \small
    \caption{Comparison of the GBC and D-WMPC approaches in simulation.}
    \label{table:simulations}
    \begin{tabular}{|ccc|rccrrrcr|} 
        \hline
        \vspace{-0.3cm}
        & & & & & & & & & & \\
        \# & $m$ & & $T$ [s] & $t_\mathrm{sol, avg}$ [s] & $t_\mathrm{sol, max}$ [s] & $l_\mathrm{avg}$ & $l_\mathrm{min}$ & $l_\mathrm{max}$ & $u_\mathrm{avg}$ &  $\Sigma u$ \\
        \hline
        \multirow{2}*{\textbf{1}} &  3 & D-WMPC &  32.50 & 0.0348$\pm$0.0060 & 0.0539 & 11.33$\pm$1.22  &  9.77 & 12.75 & 0.0007$\pm$0.0044 &   8.52 \\
        &  3 & GBC    & 115.50 & 0.0268$\pm$0.0141 & 0.0785 & 32.18$\pm$6.99  & 23.10 & 40.13 & 0.0020$\pm$0.0071 &  24.18 \\
        \multirow{2}*{\textbf{2}} &  5 & D-WMPC &  29.50 & 0.0354$\pm$0.0083 & 0.0583 &  7.94$\pm$2.71  &  4.40 & 12.69 & 0.0005$\pm$0.0035 &   9.96 \\
        &  5 & GBC    &  94.75 & 0.0278$\pm$0.0152 & 0.0980 & 26.71$\pm$3.78  & 22.66 & 33.26 & 0.0017$\pm$0.0064 &  33.43 \\
        \multirow{2}*{\textbf{3}} &  5 & D-WMPC &  44.75 & 0.0352$\pm$0.0089 & 0.0584 &  9.47$\pm$2.84  &  5.04 & 13.61 & 0.0006$\pm$0.0035 &  11.89 \\
        &  5 & GBC    & 145.25 & 0.0223$\pm$0.0146 & 0.0983 & 27.64$\pm$3.87  & 23.10 & 34.21 & 0.0017$\pm$0.0064 &  34.59 \\
        \multirow{2}*{\textbf{4}} & 10 & D-WMPC & 108.50 & 0.0341$\pm$0.0099 & 0.1007 & 15.36$\pm$4.65  &  7.87 & 22.14 & 0.0010$\pm$0.0042 &  38.47 \\
        & 10 & GBC    & 314.00 & 0.0162$\pm$0.0108 & 0.0787 & 37.65$\pm$8.67  & 23.22 & 54.79 & 0.0024$\pm$0.0076 &  94.18 \\
        \multirow{2}*{\textbf{5}} & 10 & D-WMPC &  84.25 & 0.0394$\pm$0.0115 & 0.0855 & 14.68$\pm$7.26  &  6.49 & 29.17 & 0.0009$\pm$0.0046 &  36.75 \\
        & 10 & GBC    & 310.75 & 0.0158$\pm$0.0100 & 0.0670 & 37.84$\pm$11.88 & 15.75 & 58.96 & 0.0024$\pm$0.0079 &  94.64 \\
        \multirow{2}*{\textbf{6}} & 10 & D-WMPC &  96.50 & 0.0353$\pm$0.0104 & 0.0863 & 12.40$\pm$6.27  &  2.09 & 25.84 & 0.0008$\pm$0.0041 &  31.04 \\
        & 10 & GBC    & 313.25 & 0.0164$\pm$0.0109 & 0.0820 & 40.09$\pm$7.76  & 22.45 & 50.33 & 0.0025$\pm$0.0080 & 100.27 \\
        \hline
    \end{tabular}
\end{table*}

\section{Real-world experiments}
\label{sec:real-world-experiments}
A set of real-world experiments has been performed on the DISCOWER ATMOS platform \citep{roque2026toward}, which is a free-flyer platform developed at the KTH Space Robotics Laboratory simulating frictionless flight conditions, such as those encountered in space (see \href{https://atmos.discower.io/}{https://atmos.discower.io/}).
The experimental setup is shown in Figure \ref{fig:experiment-tethered}.
A video of the experiments is available at \url{https://youtu.be/_97aHa2v_qw}.

\subsection{Experimental setup}
The proposed approach has been tested on a set of 3 ATMOS robots.
The robots operate in a 2D workspace having an approximate dimension of 4.1 $\times$ 3.0 m \citep{roque2026toward}.

For the real-world experiments, our controller is used as a high-level controller, which computes a set of trajectories corresponding to the optimizers $\bm{x}^\star_i$ of \eqref{eq:MPC} for all $i\in\mathcal{M}$, which are then tracked by lower-level controllers.
To this end, the controller is deployed as a node in ROS 2 Humble \citep{macenski2022robot}. The main repository with the implementation of the controller can be found at \href{https://github.com/gbattocletti/px4-braid-mpc}{https://github.com/gbattocletti/px4-braid-mpc}.
Due to practical constraints, all the local MPC controller are solved inside the same ROS 2 node, which runs on an offboard computer connected to the ATMOS robots via a wireless connection.
The proposed distributed control architecture is retained by only using information local to each robot, or that communicated by the other robots, when solving \eqref{eq:MPC}.
For feasibility, the collision avoidance constraints in \eqref{eq:MPC} are softened by means of slack variables, as described in Appendix \ref{appendix:slack-variables}.

The reference trajectories computed by the proposed approach are tracked by a lower-level MPC controller based on the px4-mpc (\href{https://github.com/DISCOWER/px4-mpc}{https://github.com/DISCOWER/px4-mpc}), which returns a control input vector that represents the desired force and torque to be applied to the center of mass of each robot (see \citep{roque2026toward} for more details). The px4-mpc controllers also run on the offboard computer.

Finally, the desired control inputs are received and executed by the PX4 flight controllers \citep{meier2015px4} running on board of the robots. The PX4 flight controllers, together with a motion-capture system, are also tasked with the estimation of the pose and velocities of the robots, which are then read by the ROS 2 nodes and provided as inputs to the MPC controllers, closing the control loop.
The resulting control architecture is summarized in Figure \ref{fig:control-architecture}.

\begin{figure}[t]
    \centering
    \includegraphics[width=\columnwidth]{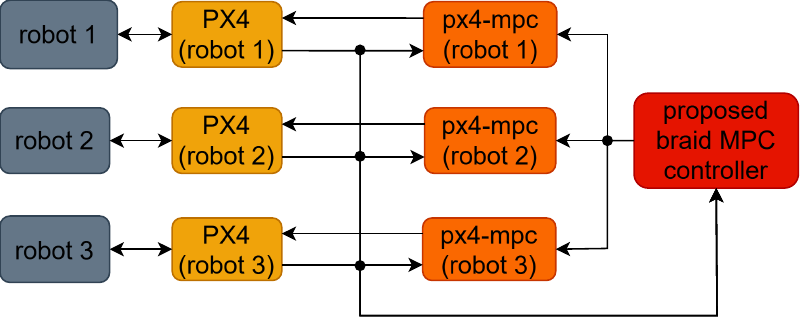}
    \caption{Control architecture for the real-world experiments in the ATMOS platform. The red box represents the ROS 2 node with the proposed controller, the orange ones the ROS 2 nodes running the px4-mpc controllers, and the yellow ones the PX4 flight controllers running on board of the robots.}
    \label{fig:control-architecture}
\end{figure}

\subsection{Experiments with ATMOS free-flyer robots}
We experimentally evaluate the proposed approach in an untethered scenario and in a tethered one, where each robot is connected by a tether to a fixed anchor point.

\subsubsection{Experiment with untethered robots}
The first experiment is performed to compare the proposed approach with the GBC baseline controller. 
To this end, we task the robots with following the specification encoded by the braid
\begin{equation}
    \label{eq:specification-1}
    \omega = \left(\sigma_2\sigma_1^{-1}\right)^4\sigma_2^{-1}\sigma_1^{-1},
\end{equation}
which corresponds to the robots looping around each other in an `8'-shaped pattern. The robots start from initial positions $x_1 = [1.0, 0.0]$, $x_2 = [2.0, 1.0]$, $x_3 = [3.0, -1.0]$, and must reach the locations $x_{\mathrm{goal}, 1} = [2.0, 0.0]$, $x_{\mathrm{goal}, 2} = [1.0, -1.0]$, $x_{\mathrm{goal}, 3} = [3.0, 1.0]$. The resulting trajectories achieved by the D-WMPC and GBC controllers are shown in Figure \ref{fig:experiment-untethered}. 
The length of the trajectories traversed by the robots are reported in Table \ref{table:experiment-untethered}, while the input force used to execute them are depicted in Figure \ref{fig:experiment-untethered-forces}.
The figure highlights the generally lower control effort required by the D-WMPC approach, which reaches the saturation point the actuators, which limits the control force to stay in the interval $[-3, 3]$~N, much less often than the GBC controller.
\begin{table}[!ht]
    \centering
    \small
    \caption{Length of the trajectories of the ATMOS robots in the experiment depicted in Figure \ref{fig:experiment-untethered}, expressed in meters.}
    \label{table:experiment-untethered}
    \begin{tabular}{cccc}
        \toprule
        & \textbf{Robot 1} & \textbf{Robot 2} & \textbf{Robot 3}\\ 
        \midrule
        \textbf{D-WMPC} & 10.59 & 9.44 & 9.66 \\
        \textbf{GBC} & 13.24 & 13.48 & 15.01 \\
        \bottomrule
    \end{tabular}
\end{table}

\begin{figure}[t]
    \centering
    \includegraphics[width=\columnwidth]{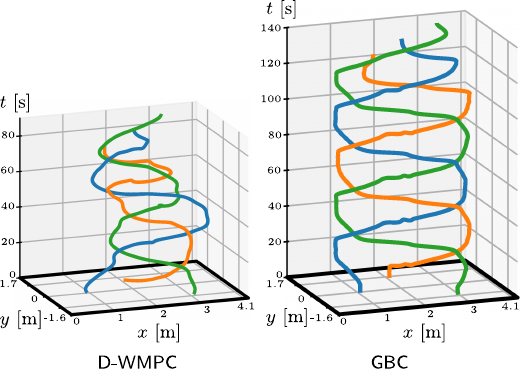}
    \caption{Experiment with the ATMOS robots. The robots are tasked with tracking the specification \eqref{eq:specification-1}. Both the sets of trajectories yield a topological braid corresponding to the specification.}
    \label{fig:experiment-untethered}
\end{figure}

\begin{figure}[t]
    \centering
    \includegraphics[width=\columnwidth]{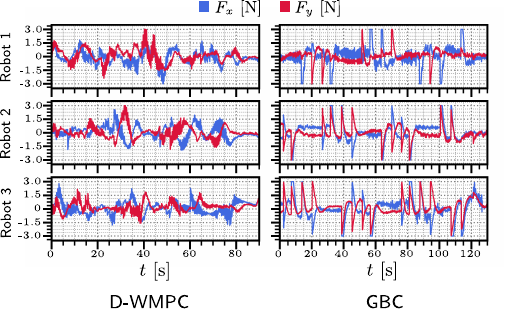}
    \caption{Input forces $F_x$ and $F_y$, expressed in N, generated by the D-WMPC (left) and GBC (right) approaches to execute the trajectories of the experiment depicted in Figure \ref{fig:experiment-untethered}.}
    \label{fig:experiment-untethered-forces}
\end{figure}

\begin{figure*}[t]
    \centering
    \begin{minipage}[c]{0.15\textwidth} 
    \centering
    \subfloat[]{
        \includegraphics[width=\linewidth]{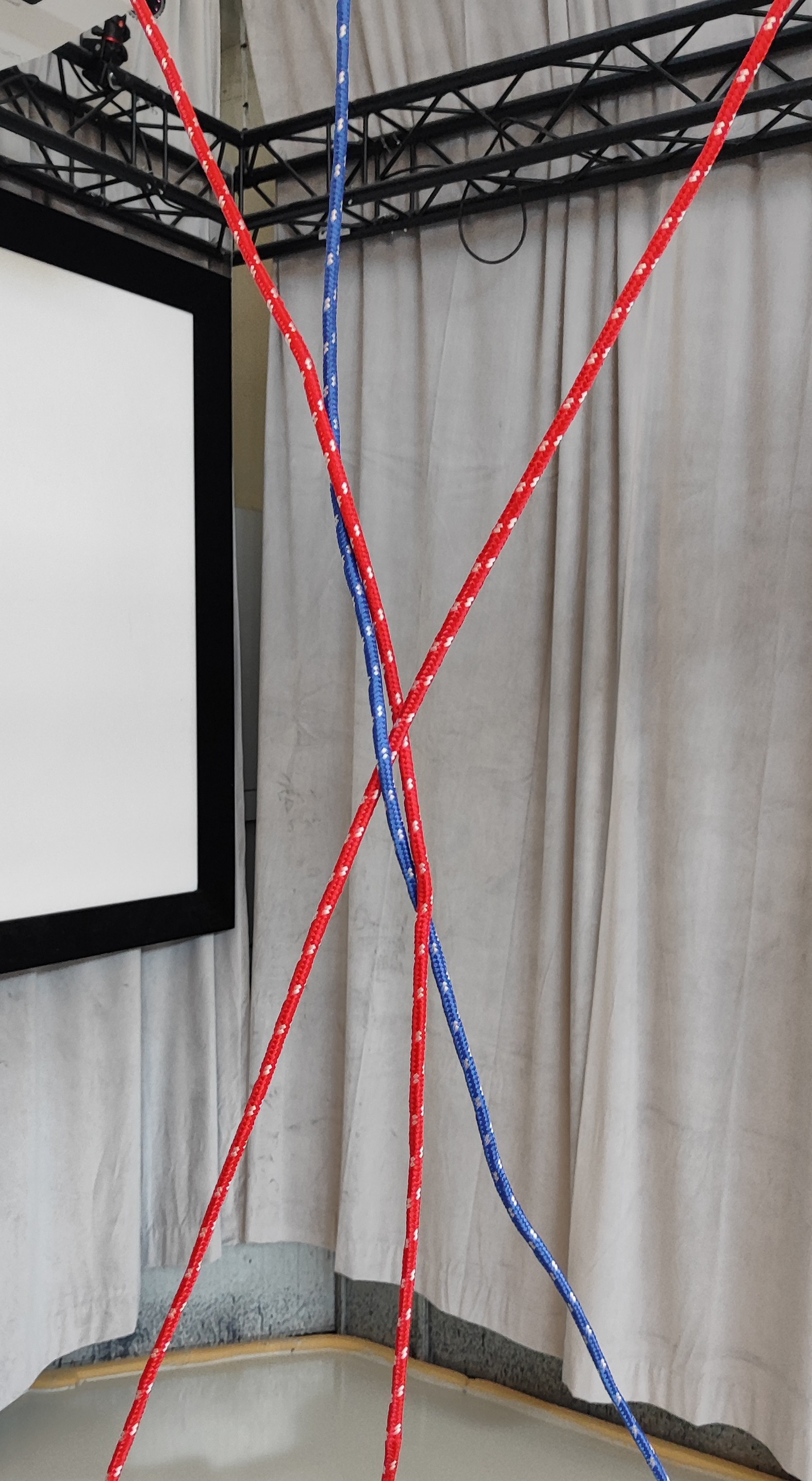}
        \label{fig:experiment-tethered-a}
        }
    \end{minipage}
    \hspace{0.01\textwidth}
    \begin{minipage}[c]{0.4\textwidth} 
        \centering
        \subfloat[]{
            \includegraphics[width=\linewidth]{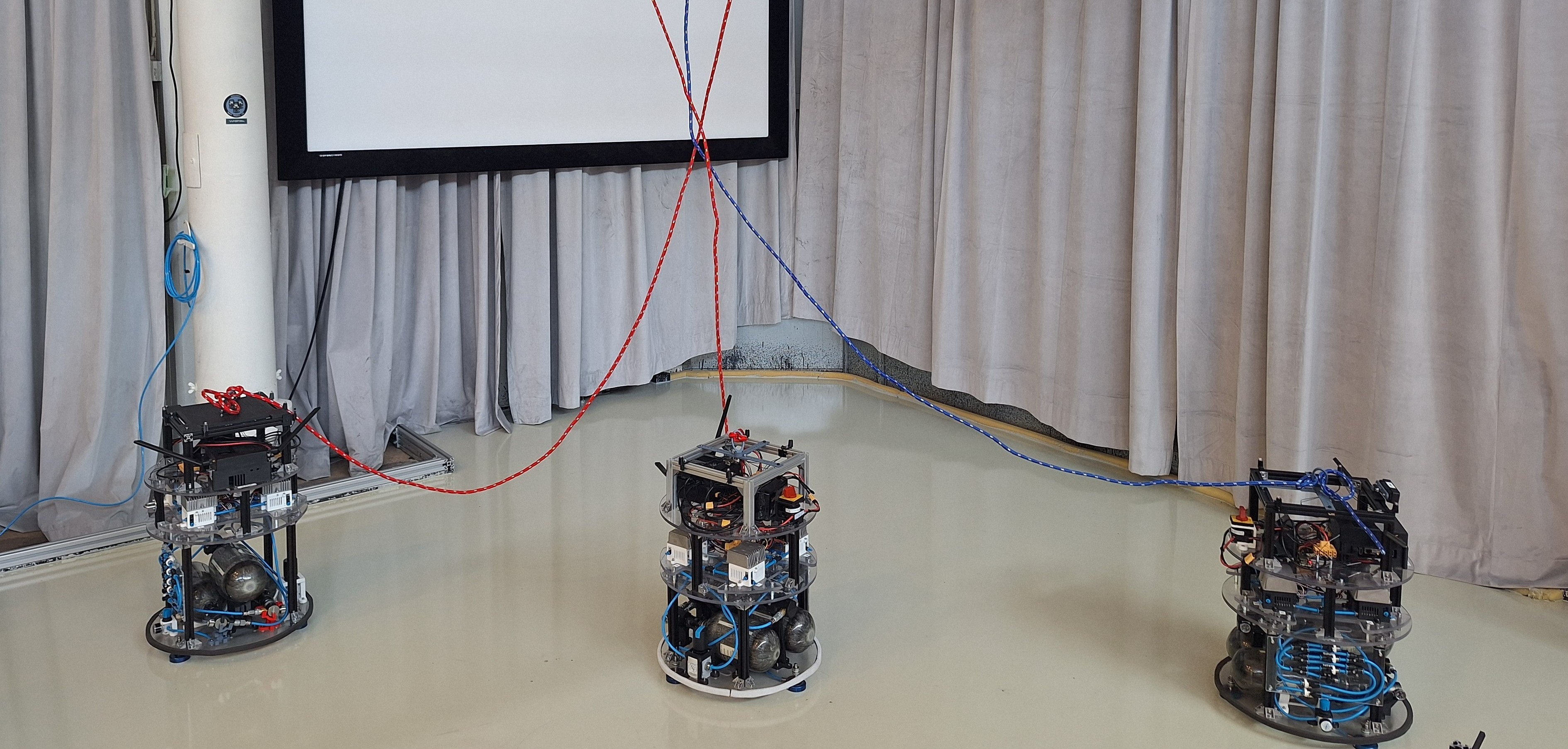}
            \label{fig:experiment-tethered-b}
            }\\
        \vspace{-0.15cm}
        \subfloat[]{
            \includegraphics[width=\linewidth]{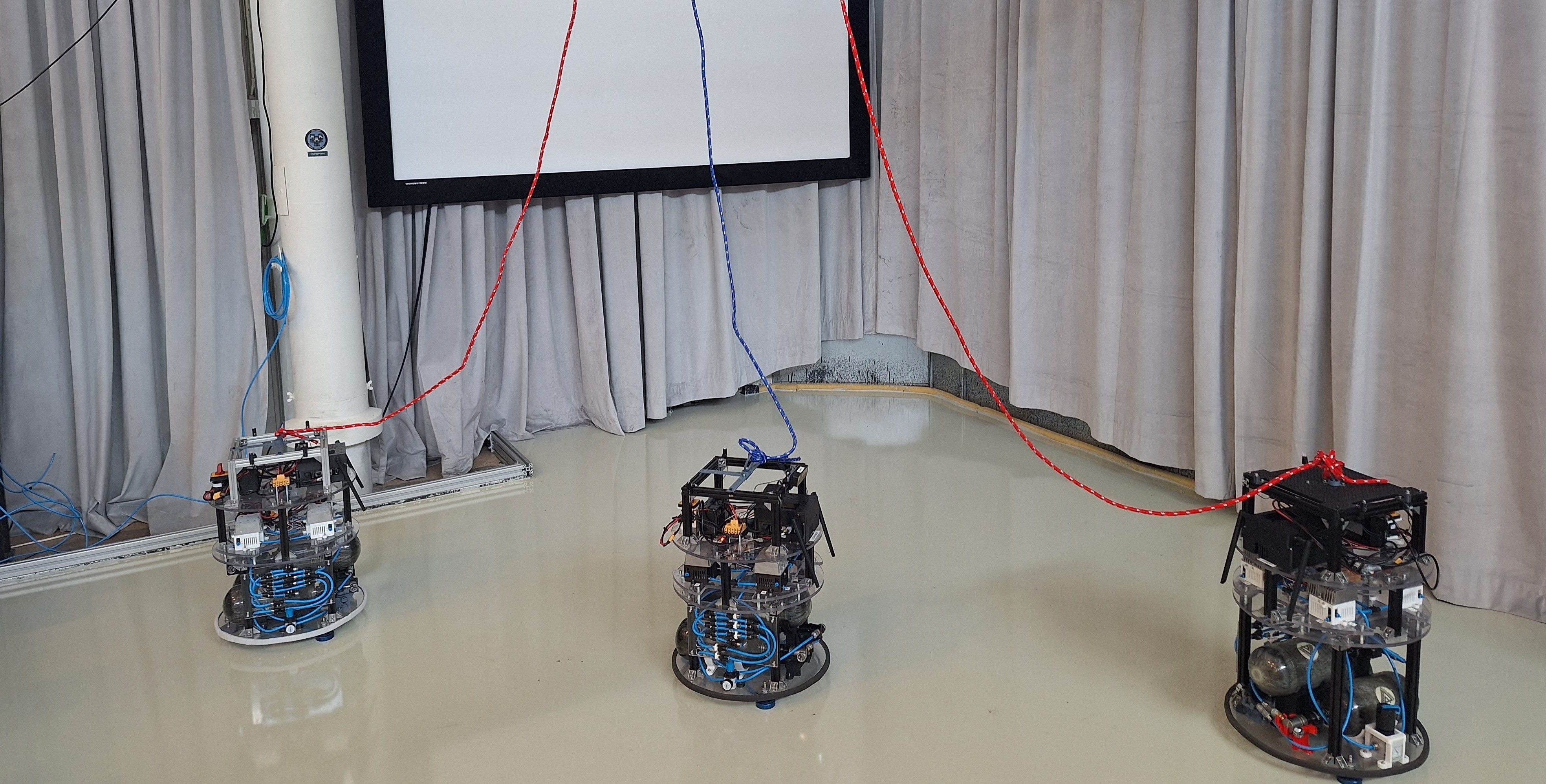}
            \label{fig:experiment-tethered-c}
            }\\
    \end{minipage}
    \hspace{-0.1cm}
    \begin{minipage}[c]{0.35\textwidth} 
        \centering
        \subfloat[]{
            \includegraphics[]{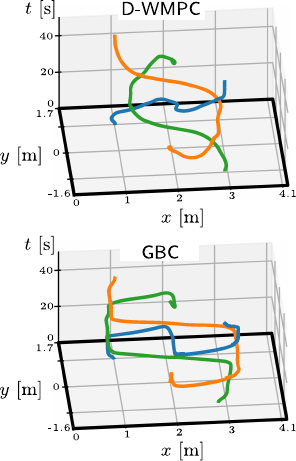}
            \label{fig:experiment-tethered-d}
            }
    \end{minipage}
    \caption{Experiment with tethered ATMOS robots. (a) A close-up of the initial knot between the tethers. (b) The robots in the initial configuration, with the tethers entangled by the knot. (c) The robots in their final configuration with the tethers free from entanglement. (d) The paths followed by the robots to disentangle the tethers, corresponding to the specification $\omega=\sigma_2\sigma_1^{-1}\sigma_2\sigma_1^{-1}$. The top plot shows the trajectories under the D-WMPC controller, the bottom one those generated by the GBC controller.}
    \label{fig:experiment-tethered}
\end{figure*}

As can be seen from the figure and from the table, the proposed approach is able to execute the specification in only 55\% of the time required by the GBC controller (85s for D-WMPC against the 140 s required by GBC).
This is achieved thanks to the computation of more efficient trajectories, on average 40\% shorter than the ones computed by the GBC approach, and thanks to the fact that the robots are not constrained to wait for others to finish executing a swap before being able to move.

\subsubsection{Experiment with tethered robots}
The second experiment presented here considers a set of tethered mobile robots. Multi-robot tethered systems represent a prime application case where topological specifications can be used for coordination purposes, usually with the objective or reaching a set of goal locations while avoiding entanglement between the tethers \citep{cao2023path, cao2025braid}. We demonstrate the capabilities of the proposed approach through a disentanglement experiment, where the robots must execute a specification designed to free a set of entangled tethers. For this purpose, we modify our problem setting to include the presence of the tethers. In particular, we take inspiration from the problem setting considered in \citep{cao2025braid}, and we connect each robot to a tether, anchored in a fixed point lying on a plane at a different height with respect to that where the robot moves. We then entangle the tethers, achieving the knot shown in Figure \ref{fig:experiment-tethered-a}. The initial configuration of the robots with the entangled tethers is shown in Figure \ref{fig:experiment-tethered-b}, where the robots are located at $x_1 = [1.0, 1.0]$, $x_2 = [2.0, 0.0]$, $x_3 = [3.0, -1.0]$. We then use an adapted version of the algorithm of \citep{cao2025braid} to compute a disentanglement specification, i.e., one that leads the robots to an entanglement-free tethered configuration. The resulting braid specification is
\begin{equation}
    \label{eq:specification-2}
    \omega = \sigma_2\sigma_1^{-1}\sigma_2\sigma_1^{-1}.
\end{equation}
The robots track this specification through the D-WMPC controller, tracing the trajectories depicted in Figure \ref{fig:experiment-tethered-d}, and finally reaching the entanglement-free configuration shown in Figure \ref{fig:experiment-tethered-c} and their goal locations $x_{\mathrm{goal}, 1} = [3.0, -1.0]$, $x_{\mathrm{goal}, 2} = [1.0, 1.0]$, $x_{\mathrm{goal}, 3} = [2.0, 0.0]$. 
Figure \ref{fig:experiment-tethered-d} also shows the trajectories generated by the GBC controller to follow the specification \eqref{eq:specification-2}. 

The comparison of the trajectories in Figure \ref{fig:experiment-tethered-d} shows how the proposed approach computes more efficient trajectories, as confirmed by the shorter length of the trajectories computed by the D-WMPC approach (see Table \ref{table:experiment-tethered}) and by the lower control effort required for the execution of the trajectories (see Figure \ref{fig:experiment-tethered-forces}).
Also in this case, the actuators require a generally lower control effort, never reaching the saturation point of the actuator, differently from the GBC controller, which often does so.
Additionally, the trajectories computed by the D-WMPC approach generally keep the robots closer together; this is often desirable in tethered robot systems, as wider movements are more likely to result in the robots pulling each other's tethers. 

\begin{figure}[t]
    \centering
    \includegraphics[width=\columnwidth]{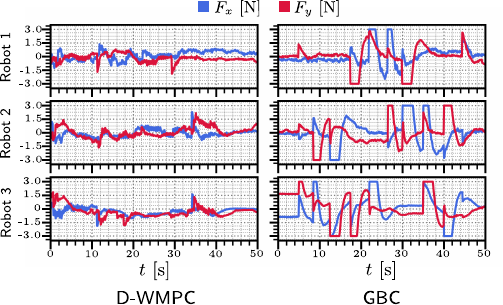}
    \caption{Input forces $F_x$ and $F_y$, expressed in N, generated by the D-WMPC (left) and GBC (right) approaches to execute the trajectories of the experiment depicted in Figure \ref{fig:experiment-tethered}.}
    \label{fig:experiment-tethered-forces}
\end{figure}

\begin{table}[!ht]
    \centering
    \small
    \caption{Length of the trajectories of the ATMOS robots in the experiment depicted in Figure \ref{fig:experiment-tethered}, expressed in meters.}
    \label{table:experiment-tethered}
    \begin{tabular}{cccc}
        \toprule
        & \textbf{Robot 1} & \textbf{Robot 2} & \textbf{Robot 3}\\ 
        \midrule
        \textbf{D-WMPC} & 3.46 & 4.80 & 4.51 \\
        \textbf{GBC} & 4.51 & 6.30 & 6.49 \\
        \bottomrule
    \end{tabular}
\end{table}

\section{Conclusions}
\label{sec:conclusions}
We have proposed a novel distributed MPC-based control approach for the distributed execution of multi-robot motion plans encoded as braid-based topological constraints.
The controller is based on the computation of trajectories to track a set of pairwise winding numbers obtained from the given topological specification, and encoding the same interaction pattern thanks to their invariance for topological braids. As a coordination tool to maintain the synchronization of the robots toward the specification, the robots use a progress variable computed locally and agreed upon through distributed consensus. 
We have demonstrated the effectiveness of the proposed approach against existing methods in simulations and in real-world experiments.

Future work will look at providing formal guarantees about the specification tracking, and at further improving the preprocessing routine to compute the target winding numbers, allowing the MPC planner expedite the execution of the specification.
Furthermore, the scalability of the algorithm will be improved by removing the requirement for all-to-all communication, and using only information about neighboring robots during the online control phase.
Additionally, the proposed approach will be extended to the 3D case, where the main challenge lies in the fact that a single topological braid is not able to fully describe the qualitative relation between the space-time trajectories of the robots, and additional topological tools are required.

\theendnotes

\begin{sm}
Supplemental material for this article is available online at \url{https://youtu.be/_97aHa2v_qw}.
\end{sm}

\begin{acks}
The authors wish to thank Prasetyo Sanjaya and Elias Krantz, from the Royal Institute of Technology, Stockholm, Sweden, for their support during the implementation of the real-world experiments in the ATMOS setup.
\end{acks}

\begin{funding}
This publication has been supported by funding from the European Union's Horizon Europe Programme under grant agreement No 101093822 (SeaClear 2.0).

\begin{dci}
The author(s) declared no potential conflicts of interest with respect to the research, authorship, and/or publication of this article.
\end{dci}

\end{funding}
\bibliographystyle{SageH}
\bibliography{references.bib}

\appendix

\section{Bound on winding numbers}
\label{appendix:bound-windings}
We prove here the linear bound \eqref{eq:windings-lipschitz} between the change in $w_{i,j}$ and the change in $\tau$. Given $\tau_1, \tau_2 \in [0,1], \tau_1<\tau_2$, we aim to find a bound on $|w_{i,j}(\tau_2) - w_{i, j}(\tau_1)|$.
Let $b_1 = \lceil B\tau_1 \rceil$ and $b_2 = \lfloor B\tau_2 \rfloor$, so that $b_1/B, (b_1+1)/B, \dots, b_2/B$ are the breakpoints in $[\tau_1, \tau_2]$, with $\tau_1 \leq b_1/B < (b_1+1)/B < \dots < (b_2-1)/B < b_2/B \leq \tau_2$.
Splitting the interval $[\tau_1, \tau_2]$ at these breakpoints and applying the triangle inequality, we have
\begin{equation}
    \begin{aligned}
        \bigg|w_{i,j}&(\tau_2) - w_{i,j}(\tau_1)\bigg|\\
        &\leq \left|w_{i,j}(\tau_2) - w_{i,j}\left(\frac{b_2}{B}\right)\right| \\
        &\qquad+ \sum_{b=b_1}^{b_2-1} \left|w_{i,j}\left(\frac{b+1}{B}\right) - w_{i,j}\left(\frac{b}{B}\right)\right| \\
        &\qquad+ \left|w_{i,j}\left(\frac{b_1}{B}\right) - w_{i,j}(\tau_1)\right|\\
        &\leq \frac{B}{2}\left(\tau_2 - \frac{b_2}{B}\right) + \frac{b_2 - b_1}{2} + \frac{B}{2}\left(\frac{b_1}{B} - \tau_1\right) \\
        &\leq \frac{B}{2}(\tau_2 - \tau_1).  
    \end{aligned}
\end{equation}
In the second inequality, each term of the sum is at most $1/2$, since by \eqref{eq:windings-linear-segments} the winding numbers change by at most $1/2$ between consecutive breakpoints. 
The first term and the last term involve partial segments, on which $w_{i,j}$ is linear by \eqref{eq:windings-linear-segments}. The change over a partial segment is the maximum change over the full segment (i.e., $1/2$) scaled by the fraction of the segment covered, i.e., $B(\tau_2 - b_2/B)$ and $B(b_1/B - \tau_1)$, respectively.
The final equality follows since $(b_2 - b_1)/2 = B/2\cdot(b_2 - b_1)/B$ and the lengths of the pieces add up to $\tau_2 - \tau_1$.
If no breakpoint lies in $[\tau_1, \tau_2]$, the same scaling argument of the partial segments gives $|w_{i,j}(\tau_2) - w_{i,j}(\tau_1)| \leq \frac{B}{2}(\tau_2 - \tau_1)$ directly.

\section{Maximum angle change over a time step}
\label{appendix:delta-tau-max}
The value of $\Delta\tau_\mathrm{max}$ can be derived considering the maximum change in the relative angle between two robots over one time step. An example is shown in Figure \ref{fig:delta-tau-max}.
\begin{figure}[t]
    \centering
    \includegraphics[width=\columnwidth]{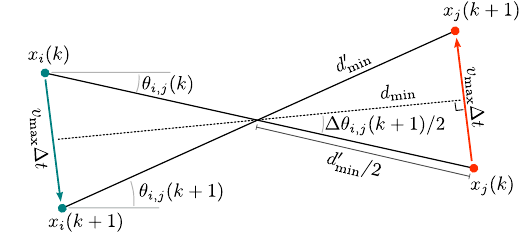}
    \caption{Maximum change in the relative heading between two robots over a time step. Two robots $i, j$ are displayed, each moving of the maximum distance they can cover in one time step, i.e., $v_\mathrm{max}\Delta t$.}
    \label{fig:delta-tau-max}
\end{figure}
As shown in the image, each robot $i, j$ can travel, in a single time step, a maximum distance equal to $v_\mathrm{max} \Delta t$.
Since we consider discrete time steps, during which the robots move in a straight line, and we require the robots to always be at a distance larger than or equal to $d_\mathrm{min}$, the initial and final positions of the robots must be at a distance of at least $d'_\mathrm{min}$, which is found as \citep{chen2024deadlock}
\begin{equation}
    d'_\mathrm{min} = \sqrt{d_\mathrm{min}^2 + \left(v_\mathrm{max}\Delta t\right)^2}.
\end{equation}
Then, the relative motion of the robots gives raise to a change in the relative heading between them corresponding to the geometric relation
\begin{equation}
    \label{eq:delta-theta-max-to-rearrange}
    \frac{d'_\mathrm{min}}{2} \sin\left(\frac{\delta(\theta_{i,j}(k+1), \theta_{i,j}(k))}{2}\right) = \frac{v_\mathrm{max} \Delta t}{2},
\end{equation}
shown in the example of Figure \ref{fig:delta-tau-max}.
Equation \eqref{eq:delta-theta-max-to-rearrange} can be rearranged to obtain
\begin{equation}
    \label{eq:delta-theta-max}
    \delta(\theta_{i,j}(k+1), \theta_{i,j}(k)) = 2\arcsin\left(\frac{v_\mathrm{max} \Delta t}{d'_\mathrm{min}}\right)
\end{equation}
We know from \eqref{eq:winding-discrete} that $\delta(\theta_{i,j}(k+1), \theta_{i,j}(k))$ corresponds to the maximum change of $\widebar{w}_{i,j}$ in a single time step:
\begin{equation}
    \label{eq:delta-w-max}
    \begin{aligned}
        \Delta \widebar{w}_{i,j} &= \frac{1}{2\pi}\delta\big(\theta_{i,j}(k+1), \theta_{i,j}(k)\big) \\
        &= \frac{1}{\pi}\arcsin\left(\frac{v_\mathrm{max} \Delta t}{d'_\mathrm{min}}\right)  
    \end{aligned}
\end{equation}

Under the assumption of linear interpolation of the winding numbers introduced in Section \ref{sec:background}, the change in $\widebar{w}_{i,j}$ is related to the change in $\tau$ through \eqref{eq:windings-lipschitz}. By combining \eqref{eq:windings-lipschitz} with \eqref{eq:delta-w-max}, the relation \eqref{eq:delta-tau-max} follows. 

\section{Updates to the MPC controllers in real-world experiments}
\label{appendix:slack-variables}
Given the multiple sources of uncertainty intrinsic in real-world experiments, such as mismatch between the real dynamic of the robot and the prediction model of the MPC controller, actuation noise, and measurement noise, feasibility of the local MPC problem is not guaranteed due to the presence of hard constraints in \eqref{eq:MPC}. For this reason, in our experiments, we soften the collision avoidance constraints via slack variables. The constraint \eqref{eq:constraint-collisions} is then replaced with
\begin{equation}
    a^\top_{i,j}(h|k)x_i(h|k) + s_{i, j}(h|k) \geq b_{i,j}(h|k)
\end{equation}
where $a_{i,j}(h|k)$ and $b_{i,j}(h|k)$ are defined as in \eqref{eq:hyperplanes-parameters}, while $s_{i, j}$ is a slack variable local to robot $i$. 
In order to ensure slack variables are non-negative, the constraint
\begin{equation}
    s_{i,j}(h|k) \geq 0, \forall h \in \{0, \ldots, K\}
\end{equation}
is added to \eqref{eq:MPC}. Additionally, slack variables are penalized in the cost function through the term
\begin{equation}
    J_{\mathrm{s}, i}(k) = \xi_\mathrm{s}\sum_{j\in\mathcal{M}\setminus \{i\}}\sum_{h=0}^{K}s^2_{i, j}(h|k)
\end{equation}
which is added to \eqref{eq:cost-function}, with $\xi_\mathrm{s} > 0$ a large weight.
Since, in general, the use of softened constraint could lead to collisions between robots, during the experiments the safety distance $d_\mathrm{min}$ in \eqref{eq:hyperplanes-parameters-b} is slightly inflated with respect to the one in the nominal case. In the experiments, this approach results in the computation of collision-free trajectories by the MPC controller.

\end{document}

%% file: notation.tex
\theoremstyle{definition}

\newtheorem*{definition*}{Definition}

\newtheorem{remark}{Remark}
\newtheorem{proposition}{Proposition}

\newtheorem{assumption}{Assumption}

\DeclareMathOperator*{\argmin}{arg\,min}

\renewcommand{\emptyset}{\text{\O}}             
\renewcommand{\epsilon}{\varepsilon}
\DeclareMathOperator{\arctantwo}{arctan2}

\makeatletter
\let\save@mathaccent\mathaccent
\newcommand*\if@single[3]{%
     \setbox0\hbox{${\mathaccent"0362{#1}}^H$}%
     \setbox2\hbox{${\mathaccent"0362{\kern0pt#1}}^H$}%
     \ifdim\ht0=\ht2 #3\else #2\fi
}
\newcommand*\rel@kern[1]{\kern#1\dimexpr\macc@kerna}
\newcommand*\widebar[1]{\@ifnextchar^{{\wide@bar{#1}{0}}}{\wide@bar{#1}{1}}}
\newcommand*\wide@bar[2]{\if@single{#1}{\wide@bar@{#1}{#2}{1}}{\wide@bar@{#1}{#2}{2}}}
\newcommand*\wide@bar@[3]{%
     \begingroup
     \def\mathaccent##1##2{%
          \let\mathaccent\save@mathaccent
          \if#32 \let\macc@nucleus\first@char \fi
          \setbox\z@\hbox{$\macc@style{\macc@nucleus}_{}$}%
          \setbox\tw@\hbox{$\macc@style{\macc@nucleus}{}_{}$}%
          \dimen@\wd\tw@
          \advance\dimen@-\wd\z@
          \divide\dimen@ 3
          \@tempdima\wd\tw@
          \advance\@tempdima-\scriptspace
          \divide\@tempdima 10
          \advance\dimen@-\@tempdima
          \ifdim\dimen@>\z@ \dimen@0pt\fi
          \rel@kern{0.6}\kern-\dimen@
          \if#31
               \overline{\rel@kern{-0.6}\kern\dimen@\macc@nucleus\rel@kern{0.4}\kern\dimen@}%
               \advance\dimen@0.4\dimexpr\macc@kerna
               \let\final@kern#2%
               \ifdim\dimen@<\z@ \let\final@kern1\fi
               \if\final@kern1 \kern-\dimen@\fi
          \else
               \overline{\rel@kern{-0.6}\kern\dimen@#1}%
          \fi
     }%
     \macc@depth\@ne
     \let\math@bgroup\@empty \let\math@egroup\macc@set@skewchar
     \mathsurround\z@ \frozen@everymath{\mathgroup\macc@group\relax}%
     \macc@set@skewchar\relax
     \let\mathaccentV\macc@nested@a
     \if#31
     \macc@nested@a\relax111{#1}%
     \else
          \def\gobble@till@marker##1\endmarker{}%
          \futurelet\first@char\gobble@till@marker#1\endmarker
          \ifcat\noexpand\first@char A\else
               \def\first@char{}%
          \fi
          \macc@nested@a\relax111{\first@char}%
     \fi
     \endgroup
}
\makeatother

%% file: references.bib
@article{artin1947theory,
    author  = {Artin, Emil},
    journal = {Annals of Mathematics},
    number  = {1},
    pages   = {101--126},
    title   = {Theory of braids},
    volume  = {48},
    year    = {1947}
}

@article{baselga2026shine,
    author  = {Diego Martinez-Baselga and Oscar de Groot and Luzia Knoedler and Luis Riazuelo and Javier Alonso-Mora and Luis Montano},
    journal = {The International Journal of Robotics Research},
    number  = {1},
    pages   = {60-79},
    title   = {{SHINE}: Social homology identification for navigation in crowded environments},
    volume  = {45},
    year    = {2026}
}

@article{berger2001topological,
    author  = {Berger, Mitchell A},
    journal = {Letters in Mathematical Physics},
    number  = {3},
    pages   = {181--192},
    title   = {Topological invariants in braid theory},
    volume  = {55},
    year    = {2001}
}

@article{bhattacharya2010search,
    author  = {Bhattacharya, S.},
    journal = {Proceedings of the AAAI Conference on Artificial Intelligence},
    number  = {1},
    pages   = {1230--1237},
    title   = {Search-Based Path Planning with Homotopy Class Constraints},
    volume  = {24},
    year    = {2010}
}

@article{bhattacharya2012topological,
    author  = {Bhattacharya, S. and Likhachev, M. and Kumar, V.},
    journal = {Autonomous Robots},
    number  = {3},
    pages   = {273--290},
    title   = {Topological constraints in search-based robot path planning},
    volume  = {33},
    year    = {2012}
}

@article{bhattacharya2018path,
    author  = {Bhattacharya, S. and Ghrist, R.},
    issue   = {3},
    journal = {Annals of Mathematics and Artificial Intelligence},
    pages   = {139--160},
    title   = {Path homotopy invariants and their application to optimal trajectory planning},
    volume  = {84},
    year    = {2018}
}

@article{bhattacharya2021guest,
    author  = {Bhattacharya, Subhrajit and Pokorny, Florian T and Choset, Howie},
    journal = {Autonomous Robots},
    number  = {5},
    pages   = {611--612},
    title   = {Guest Editorial: Special issue on “{Topological} methods in robotics”},
    volume  = {45},
    year    = {2021}
}

@article{brito2019model,
    author  = {Brito, Bruno and Floor, Boaz and Ferranti, Laura and Alonso-Mora, Javier},
    journal = {IEEE Robotics and Automation Letters},
    number  = {4},
    pages   = {4459--4466},
    title   = {Model predictive contouring control for collision avoidance in unstructured dynamic environments},
    volume  = {4},
    year    = {2019}
}

@inproceedings{cao2023path,
    author    = {Cao, M. and Cao, K. and Yuan, S. and Liu, K. and Wong, Y. L. and Xie, L.},
    booktitle = {Robotics: Science and Systems XIX},
    pages     = {1--10},
    title     = {Path Planning for Multiple Tethered Robots Using Topological Braids},
    volume    = {19},
    year      = {2023}
}

@article{cao2025braid,
    author  = {Muqing Cao and Kun Cao and Xinhang Xu and Shenghai Yuan and Yan Loi Wong and Lihua Xie},
    journal = {The International Journal of Robotics Research},
    number  = {0},
    title   = {Braid-based entanglement-free trajectory planning for multiple tethered robots},
    volume  = {0},
    year    = {2025}
}

@article{chen2020guaranteed,
    author  = {Chen, Yuxiao and Singletary, Andrew and Ames, Aaron D},
    journal = {IEEE Control Systems Letters},
    number  = {1},
    pages   = {127--132},
    title   = {Guaranteed obstacle avoidance for multi-robot operations with limited actuation: A control barrier function approach},
    volume  = {5},
    year    = {2020}
}

@article{chen2024deadlock,
    author  = {Chen, Yuda and Guo, Meng and Li, Zhongkui},
    journal = {IEEE Transactions on Automatic Control},
    number  = {9},
    pages   = {6058--6073},
    title   = {Deadlock resolution and recursive feasibility in {MPC}-based multirobot trajectory generation},
    volume  = {69},
    year    = {2024}
}

@article{dantonio2026hitch,
    author  = {D'Antonio, Diego S. and Wu, Tongshu and Bhattacharya, Subhrajit and Saldaña, David},
    journal = {IEEE Transactions on Robotics},
    number  = {},
    pages   = {1769--1783},
    title   = {From Hitch to Lift: Autonomous Cable Interlacing by Multi-{UAV} Teams for Aerial Grasping and Transportation},
    volume  = {42},
    year    = {2026}
}

@article{degroot2024topology,
    author  = {De Groot, Oscar and Ferranti, Laura and Gavrila, Dariu M and Alonso-Mora, Javier},
    journal = {IEEE Transactions on Robotics},
    pages   = {110--126},
    title   = {Topology-driven parallel trajectory optimization in dynamic environments},
    volume  = {41},
    year    = {2024}
}

@article{diaz2017multirobot,
    author  = {Diaz-Mercado, Yancy and Egerstedt, Magnus},
    journal = {IEEE Transactions on Robotics},
    number  = {6},
    pages   = {1375--1385},
    title   = {Multirobot mixing via braid groups},
    volume  = {33},
    year    = {2017}
}

@inproceedings{diazmercado2015correct,
    author    = {Diaz-Mercado, Yancy and Jones, Austin and Belta, Calin and Egerstedt, Magnus},
    booktitle = {2015 54th IEEE Conference on Decision and Control (CDC)},
    number    = {},
    pages     = {221--226},
    title     = {Correct-by-construction control synthesis for multi-robot mixing},
    volume    = {},
    year      = {2015}
}

@inproceedings{dong2025homotopy,
    author    = {Dong, Haoze and Guo, Meng and He, Chengyi and Li, Zhongkui},
    booktitle = {2025 IEEE/RSJ International Conference on Intelligent Robots and Systems (IROS)},
    pages     = {6448--6455},
    title     = {Homotopy-aware multi-agent navigation via distributed model predictive control},
    year      = {2025}
}

@book{epstein1992word,
    author    = {Epstein, David BA},
    publisher = {AK Peters/CRC Press},
    title     = {Word processing in groups},
    year      = {1992}
}

@article{garside1969braid,
    author  = {Garside, Frank A},
    journal = {The Quarterly Journal of Mathematics},
    number  = {1},
    pages   = {235--254},
    title   = {The braid group and other groups},
    volume  = {20},
    year    = {1969}
}

@inproceedings{grigoriev1998polytime,
    author    = {Grigoriev, Dima and Slissenko, Anatol},
    booktitle = {Proceedings of the 1998 International Symposium on Symbolic and Algebraic Computation},
    pages     = {17--24},
    title     = {Polytime algorithm for the shortest path in a homotopy class amidst semi-algebraic obstacles in the plane},
    year      = {1998}
}

@article{he2022homotopy,
    author  = {He, Wenbo and Huang, Yunshen and Wang, Jie and Zeng, Shen},
    journal = {IEEE Control Systems Letters},
    pages   = {1045--1050},
    title   = {Homotopy method for optimal motion planning with homotopy class constraints},
    volume  = {7},
    year    = {2022}
}

@article{imran2026homotopy,
    author  = {Imran, M. I. I. S. and Peters, L. and Khayyat, M. and Arrigoni, S. and Braghin, F. and Ferranti, L.},
    journal = {IEEE Robotics and Automation Letters},
    number  = {5},
    pages   = {5962--5969},
    title   = {Homotopy-Guided Potential Games for Congestion-Aware Navigation},
    volume  = {11},
    year    = {2026}
}

@book{kassel2008braid,
    author    = {Kassel, Christian and Turaev, Vladimir},
    publisher = {Springer Science \& Business Media},
    title     = {Braid groups},
    year      = {2008}
}

@inproceedings{kim2014path,
    author    = {Kim, S. and Bhattacharya, S. and Kumar, V.},
    booktitle = {2014 IEEE International Conference on Robotics and Automation (ICRA)},
    pages     = {1132--1139},
    title     = {Path planning for a tethered mobile robot},
    year      = {2014}
}

@book{lee2010introduction,
    author    = {Lee, J.},
    publisher = {Springer},
    title     = {Introduction to Topological Manifolds},
    year      = {2010}
}

@article{liu2024reasoning,
    author  = {Liu, Haochen and Chen, Li and Qiao, Yu and Lv, Chen and Li, Hongyang},
    journal = {Advances in Neural Information Processing Systems},
    pages   = {92605--92637},
    title   = {Reasoning multi-agent behavioral topology for interactive autonomous driving},
    volume  = {37},
    year    = {2024}
}

@article{lyu2019multivehicle,
    author  = {Lyu, Yang and Hu, Jinwen and Chen, Ben M and Zhao, Chunhui and Pan, Quan},
    journal = {IEEE Transactions on Cybernetics},
    number  = {5},
    pages   = {2651--2662},
    title   = {Multivehicle flocking with collision avoidance via distributed model predictive control},
    volume  = {51},
    year    = {2019}
}

@inproceedings{ma2025topology,
    author    = {Ma, Changjia and Zhao, Yi and Gan, Zhongxue and Gao, Bingzhao and Ding, Wenchao},
    booktitle = {2025 IEEE/RSJ International Conference on Intelligent Robots and Systems (IROS)},
    pages     = {10220--10227},
    title     = {Topology-Driven Trajectory Optimization for Modelling Controllable Interactions Within Multi-Vehicle Scenario},
    year      = {2025}
}

@article{macenski2022robot,
    author  = {Steven Macenski and Tully Foote and Brian Gerkey and Chris Lalancette and William Woodall},
    journal = {Science Robotics},
    number  = {66},
    pages   = {eabm6074},
    title   = {{Robot Operating System 2}: Design, architecture, and uses in the wild},
    volume  = {7},
    year    = {2022}
}

@inproceedings{mavrogiannis2017socially,
    author    = {Mavrogiannis, Christoforos and Blukis, Valts and Knepper, Ross A},
    booktitle = {2017 IEEE/RSJ International Conference on Intelligent Robots and Systems (IROS)},
    pages     = {6817--6824},
    title     = {Socially competent navigation planning by deep learning of multi-agent path topologies},
    year      = {2017}
}

@inproceedings{mavrogiannis2018multi,
    author    = {Mavrogiannis, Christoforos I and Knepper, Ross A},
    booktitle = {International Workshop on the Algorithmic Foundations of Robotics},
    pages     = {744--761},
    title     = {Multi-agent trajectory prediction and generation with topological invariants enforced by hamiltonian dynamics},
    year      = {2018}
}

@article{mavrogiannis2019multiagent,
    author  = {Christoforos Mavrogiannis and Ross A Knepper},
    journal = {The International Journal of Robotics Research},
    number  = {2-3},
    pages   = {338--356},
    title   = {Multi-agent path topology in support of socially competent navigation planning},
    volume  = {38},
    year    = {2019}
}

@inproceedings{mavrogiannis2020decentralized,
    author    = {Mavrogiannis, Christoforos and Knepper, Ross A},
    booktitle = {Algorithmic Foundations of Robotics XII: Proceedings of the Twelfth Workshop on the Algorithmic Foundations of Robotics},
    pages     = {880--895},
    title     = {Decentralized multi-agent navigation planning with braids},
    year      = {2020}
}

@article{mavrogiannis2021hamiltonian,
    author  = {Mavrogiannis, Christoforos and Knepper, Ross A},
    journal = {The International Journal of Robotics Research},
    number  = {10-11},
    pages   = {1234--1254},
    title   = {Hamiltonian coordination primitives for decentralized multiagent navigation},
    volume  = {40},
    year    = {2021}
}

@article{mavrogiannis2022winding,
    author  = {Mavrogiannis, Christoforos and Balasubramanian, Krishna and Poddar, Sriyash and Gandra, Anush and Srinivasa, Siddhartha S},
    journal = {IEEE Robotics and Automation Letters},
    number  = {1},
    pages   = {121--128},
    title   = {Winding through: Crowd navigation via topological invariance},
    volume  = {8},
    year    = {2022}
}

@inproceedings{mavrogiannis2023implicit,
    author    = {Mavrogiannis, Christoforos and DeCastro, Jonathan A. and Srinivasa, Siddhartha},
    booktitle = {Algorithmic Foundations of Robotics XV},
    pages     = {368--384},
    title     = {Implicit Multiagent Coordination at Uncontrolled Intersections via Topological Braids},
    year      = {2023}
}

@article{mavrogiannis2024abstracting,
    author  = {Mavrogiannis, Christoforos and DeCastro, Jonathan A and Srinivasa, Siddhartha S},
    journal = {The International Journal of Robotics Research},
    number  = {9},
    pages   = {1299--1321},
    title   = {Abstracting road traffic via topological braids: Applications to traffic flow analysis and distributed control},
    volume  = {43},
    year    = {2024}
}

@article{mcbeth2023scalable,
    author  = {McBeth, Courtney and Motes, James and Uwacu, Diane and Morales, Marco and Amato, Nancy M},
    journal = {IEEE Robotics and Automation Letters},
    number  = {11},
    pages   = {6867--6874},
    title   = {Scalable multi-robot motion planning for congested environments with topological guidance},
    volume  = {8},
    year    = {2023}
}

@inproceedings{meier2015px4,
    author    = {Meier, Lorenz and Honegger, Dominik and Pollefeys, Marc},
    booktitle = {2015 IEEE International Conference on Robotics and Automation (ICRA)},
    number    = {},
    pages     = {6235--6240},
    title     = {{PX4:} A node-based multithreaded open source robotics framework for deeply embedded platforms},
    volume    = {},
    year      = {2015}
}

@misc{nakao2025symmetry,
    author       = {Nakao, Tomoki and Kasaura, Kazumi and Kozuno, Tadashi},
    howpublished = {\emph{arXiv preprint arXiv:2511.15239}},
    title        = {Symmetry-Breaking in Multi-Agent Navigation: Winding Number-Aware {MPC} with a Learned Topological Strategy},
    year         = {2025}
}

@inproceedings{narayanan2013planning,
    author    = {Narayanan, Venkatraman and Vernaza, Paul and Likhachev, Maxim and LaValle, Steven M},
    booktitle = {2013 IEEE International Conference on Robotics and Automation},
    pages     = {431--437},
    title     = {Planning under topological constraints using beam-graphs},
    year      = {2013}
}

@article{park2015homotopy,
    author  = {Park, Junghee and Karumanchi, Sisir and Iagnemma, Karl},
    journal = {IEEE Transactions on Robotics},
    number  = {5},
    pages   = {1101--1115},
    title   = {Homotopy-based divide-and-conquer strategy for optimal trajectory planning via mixed-integer programming},
    volume  = {31},
    year    = {2015}
}

@article{park2022online,
    author  = {Park, Jungwon and Kim, Dabin and Kim, Gyeong Chan and Oh, Dahyun and Kim, H Jin},
    journal = {IEEE Robotics and Automation Letters},
    number  = {2},
    pages   = {4869--4876},
    title   = {Online distributed trajectory planning for quadrotor swarm with feasibility guarantee using linear safe corridor},
    volume  = {7},
    year    = {2022}
}

@article{paterson1991set,
    author  = {Paterson, MS and Razborov, Alexander A.},
    journal = {Journal of Algorithms},
    number  = {3},
    pages   = {393--408},
    title   = {The set of minimal braids is co-{NP}-complete},
    volume  = {12},
    year    = {1991}
}

@article{pokorny2016topological,
    author  = {Pokorny, F. T. and Hawasly, M. and Ramamoorthy, S.},
    journal = {The International Journal of Robotics Research},
    number  = {1-3},
    pages   = {204--223},
    title   = {Topological trajectory classification with filtrations of simplicial complexes and persistent homology},
    volume  = {35},
    year    = {2016}
}

@article{roque2026toward,
    author  = {Roque, Pedro and Phodapol, Sujet and Krantz, Elias and Lim, Jaeyoung and Verhagen, Joris and Jiang, Frank J. and Dörner, David and Mao, Huina and Tibert, Gunnar and Siegwart, Roland and Stenius, Ivan and Tumova, Jana and Fuglesang, Christer and Dimarogonas, Dimos V.},
    journal = {IEEE Transactions on Field Robotics},
    number  = {},
    pages   = {141--161},
    title   = {Toward Open-Source and Modular Space Systems With {ATMOS}},
    volume  = {3},
    year    = {2026}
}

@article{sahin2024topo,
    author  = {Sahin, Alp and Bhattacharya, Subhrajit},
    journal = {IEEE Transactions on Robotics},
    pages   = {20--41},
    title   = {Topo-geometrically distinct path computation using neighborhood-augmented graph, and its application to path planning for a tethered robot in {3D}},
    volume  = {41},
    year    = {2024}
}

@inproceedings{tanner2013online,
    author    = {Tanner, M. M. and Burdick, J. W. and Nesnas, I. A. D.},
    booktitle = {2013 IEEE International Conference on Robotics and Automation},
    number    = {},
    pages     = {5557--5564},
    title     = {Online motion planning for tethered robots in extreme terrain},
    volume    = {},
    year      = {2013}
}

@article{teshnizi2021motion,
    author  = {Teshnizi, R. H. and Shell, D. A.},
    journal = {Autonomous Robots},
    number  = {5},
    pages   = {693--707},
    title   = {Motion planning for a pair of tethered robots},
    volume  = {45},
    year    = {2021}
}

@article{thiffeault2010braids,
    author  = {Thiffeault, Jean-Luc},
    journal = {Chaos: An Interdisciplinary Journal of Nonlinear Science},
    number  = {1},
    title   = {Braids of entangled particle trajectories},
    volume  = {20},
    year    = {2010}
}

@misc{thiffeault2014braidlab,
    author       = {Jean-Luc Thiffeault and Marko Budi\v{s}i\'{c}},
    howpublished = {\emph{arXiv preprint arXiv:2511.15239}},
    note         = {Version 3.2.6},
    title        = {Braidlab: A software package for braids and loops},
    year         = {2013--2026}
}

@book{thiffeault2022braids,
    author    = {Thiffeault, Jean-Luc},
    publisher = {Springer},
    title     = {Braids and Dynamics},
    year      = {2022}
}

@article{wang2023coordination,
    author  = {Wang, Xiaolong and Sahin, Alp and Bhattacharya, Subhrajit},
    journal = {Journal of Intelligent \& Robotic Systems},
    number  = {3},
    pages   = {50},
    title   = {Coordination-free multi-robot path planning for congestion reduction using topological reasoning},
    volume  = {108},
    year    = {2023}
}

@article{wiltz2025parallelized,
    author  = {Wiltz, Adrian and Chen, Fei and Dimarogonas, Dimos V},
    journal = {Automatica},
    pages   = {111952},
    title   = {Parallelized robust distributed model predictive control in the presence of coupled state constraints},
    volume  = {171},
    year    = {2025}
}

@article{wong2025generating,
    author  = {Wong, Rufus Cheuk Yin and Sewlia, Mayank and Wiltz, Adrian and Dimarogonas, Dimos V.},
    journal = {IEEE Robotics and Automation Letters},
    number  = {12},
    pages   = {12453--12460},
    title   = {Generating and Optimizing Topologically Distinct Guesses for Mobile Manipulator Path Planning With Path Constraints},
    volume  = {10},
    year    = {2025}
}

@inproceedings{wu2020multi,
    author    = {Wu, Wenying and Bhattacharya, Subhrajit and Prorok, Amanda},
    booktitle = {2020 IEEE International Conference on Robotics and Automation (ICRA)},
    pages     = {9809--9815},
    title     = {Multi-robot path deconfliction through prioritization by path prospects},
    year      = {2020}
}

@misc{yang2022efficient,
    author       = {Yang, T. and Huang, L. and Wang, Y. and Xiong, R.},
    howpublished = {\emph{arXiv preprint arXiv:2207.13604}},
    title        = {Efficient Search of the $k$ Shortest Non-Homotopic Paths by Eliminating Non-$k$-Optimal Topologies},
    year         = {2022}
}

@inproceedings{yershov2013continuous,
    author    = {Yershov, D.  S. and Vernaza, P. and LaValle, S. M.},
    booktitle = {2013 IEEE International Conference on Robotics and Automation},
    number    = {},
    pages     = {5551--5556},
    title     = {Continuous planning with winding constraints using optimal heuristic-driven front propagation},
    volume    = {},
    year      = {2013}
}

@article{zhang2019planning,
    author  = {Zhang, X. and Pham, Q.-C.},
    journal = {Robotics and Autonomous Systems},
    pages   = {189--203},
    title   = {Planning coordinated motions for tethered planar mobile robots},
    volume  = {118},
    year    = {2019}
}

@article{zhou2017fast,
    author  = {Zhou, Dingjiang and Wang, Zijian and Bandyopadhyay, Saptarshi and Schwager, Mac},
    journal = {IEEE Robotics and Automation Letters},
    number  = {2},
    pages   = {1047--1054},
    title   = {Fast, on-line collision avoidance for dynamic vehicles using buffered {Voronoi} cells},
    volume  = {2},
    year    = {2017}
}
